\documentclass[lettersize,journal]{IEEEtran}
\usepackage{amsmath,amsfonts}
\usepackage[ruled,vlined]{algorithm2e}

\SetKwInput{KwIn}{Input}
\SetKwInput{KwOut}{Output}
\usepackage{array}
\usepackage[skip=2pt]{caption}
\DeclareCaptionFormat{myformat}{\fontsize{8}{9}\selectfont#1#2#3}
\usepackage[caption=false,font=normalsize,labelfont=sf,textfont=sf]{subfig}
\usepackage{textcomp}
\usepackage{stfloats}
\usepackage{url}
\usepackage{verbatim}
\usepackage{multirow} 
\usepackage{array} 
\usepackage{cite}
\usepackage{soul}
\usepackage{booktabs}

\usepackage[dvipsnames]{xcolor}
\usepackage{subcaption}

\usepackage{subfig}
\usepackage{amsmath}
\usepackage{amssymb}
\usepackage{hyperref}
\usepackage{gensymb}
\hypersetup{
    colorlinks=true,
    linkcolor=blue,
    citecolor=blue,
    filecolor=magenta,      
    urlcolor=blue,
    pdfpagemode=FullScreen,
    }

\usepackage [english]{babel}
\usepackage [autostyle, english = american]{csquotes}
\MakeOuterQuote{"}
\usepackage{orcidlink}
\usepackage{amsthm}
\usepackage[english]{babel}
\newtheorem{theorem}{Theorem}

\newtheorem{lemma}{Lemma}
\newtheorem{definition}{Definition}
\renewcommand\hl[1]{#1}

\begin{document}
\newcommand\mdoubleplus{\mathbin{+\mkern-10mu+}}

\title{ConvXformer: Differentially Private Hybrid ConvNeXt-Transformer for Inertial Navigation}
\author{Omer Tariq\,\orcidlink{0000-0002-1771-6166}, Muhammad Bilal\,\orcidlink{0000-0003-4221-0877},~\IEEEmembership{~Senior Member, IEEE}, Muneeb Ul Hassan\,\orcidlink{0000-0001-5109-9547}, ~\IEEEmembership{Member, IEEE}, \\Dongsoo Han,~\IEEEmembership{~Senior Member, IEEE}, Jon Crowcroft\,\orcidlink{0000-0002-7013-0121},~\IEEEmembership{~Fellow, IEEE}
\thanks{Omer Tariq and Dongsoo Han are with the School of Computing, Korea Advanced Institute of Science and Technology, Daejeon, 34141, South Korea. (E-mail: omertariq@kaist.ac.kr, ddsshhan@kaist.ac.kr)}
\thanks{Muhammad Bilal is with the School of Computing and Communications Lancaster University, Lancaster LA1 4YW, United Kingdom (Corresponding Author E-mail: m.bilal@ieee.org)}
\thanks{Muneeb Ul Hassan is with the School of Information Technology, Deakin University, Australia (E-mail: muneebmh1@gmail.com)}
\thanks{Jon Crowcroft is with the Department of Computer Science and Technology, University of Cambridge, CB3 0FD Cambridge, U.K. (e-mail: jon.crowcroft@cl.cam.ac.uk).}

}


\maketitle

\begin{abstract}
Data-driven inertial sequence learning has revolutionized navigation in GPS-denied environments, offering superior odometric resolution compared to traditional Bayesian methods. However, deep learning-based inertial tracking systems remain vulnerable to privacy breaches that can expose sensitive training data. \hl{Existing differential privacy solutions often compromise model performance by introducing excessive noise, particularly in high-frequency inertial measurements.}
In this article, we propose ConvXformer, a hybrid architecture that fuses ConvNeXt blocks with Transformer encoders in a hierarchical structure for robust inertial navigation. We propose an efficient differential privacy mechanism incorporating adaptive gradient clipping and gradient-aligned noise injection (GANI) to protect sensitive information while ensuring model performance. Our framework leverages truncated singular value decomposition for gradient processing, enabling precise control over the privacy-utility trade-off. Comprehensive performance evaluations on benchmark datasets (OxIOD, RIDI, RoNIN) demonstrate that ConvXformer surpasses state-of-the-art methods, achieving \hl{more than 40\%} improvement in positioning accuracy while ensuring $(\epsilon,\delta)$-differential privacy guarantees. To validate real-world performance, we introduce the Mech-IO dataset, collected from the mechanical engineering building at KAIST, where intense magnetic fields from industrial equipment induce significant sensor perturbations.
This demonstrated robustness under severe environmental distortions makes our framework well-suited for secure and intelligent navigation in cyber-physical systems.

\end{abstract}


\section{Introduction}
Inertial Navigation Systems (INS) are pivotal for cost-effective, energy-efficient navigation, providing independence from external signals by processing initial conditions to determine position, velocity, and orientation. This autonomy ensures resilience against jamming and spoofing, enabling secure applications in intelligent transportation, healthcare infrastructure, and virtual reality \cite{DNN-Survey}. Despite their promise, low-cost Inertial Measurement Units (IMUs) suffer from significant challenges like environmental noise, temperature drift, and bias instability, which limit their practical deployment in consumer applications.
Traditional inertial navigation relies on triple integration for position computation; however, this process compounds stochastic noise, leading to exponential error growth that restricts reliable operation to only 1–2 minutes without external corrections~\cite{intro5}. While high-end sensors mitigate this drift, their bulk and cost make them unsuitable for pedestrian applications. Domain-specific methods such as gait-based Pedestrian Dead Reckoning (PDR) and Zero-Velocity Updates (ZUPTs) improve accuracy but remain sensitive to irregular gaits and false detections \cite{intro5}.

By learning motion patterns from large-scale datasets, neural networks have advanced data-driven inertial navigation for pedestrian tracking, outperforming traditional methods through dynamic error compensation and improved trajectory estimation from raw IMU data~\cite{DNN-Survey}.
\IEEEpubidadjcol
Despite achieving 1.5–3.0 m accuracy on benchmarks like RIDI and RoNIN, current architectures face key limitations in real-world use. CNNs effectively extract local temporal patterns but fail to model long-range dependencies needed for drift correction, while Transformers capture global relationships but incur quadratic complexity and are sensitive to noise-induced attention errors.

The deployment of neural inertial navigation systems poses serious privacy risks that existing methods overlook. Trajectory data expose spatiotemporal mobility patterns that compromise user anonymity~\cite{intro6}. Adversaries can exploit inertial sensor data with deep learning to reconstruct sub-meter trajectories, bypassing GPS and enabling de-anonymization through temporal clustering and mobility database matching~\cite{tmc-p}. In healthcare and surveillance, such breaches reveal sensitive behaviors, medical states, and personal routines. Conventional defenses are insufficient: k-anonymity~\cite{intro7, k-anony} fails against informed attackers, while cryptographic methods~\cite{crypt_loc} are too computationally expensive for mobile use.

Existing differential privacy mechanisms face severe utility–privacy trade-offs that limit their practicality for inertial navigation. Standard Differentially Private Stochastic Gradient Descent (DP-SGD)~\cite{DP-SGD} injects uniform noise that disrupts the hierarchical structure of temporal sensor features, causing 30–40\% accuracy degradation. This uniform noise fails to reflect the structured sensitivity of inertial data across temporal scales and sensor modalities. In Transformer architectures, DP-SGD further distorts self-attention, producing attention distraction—where high-variance tokens dominate attention, degrading convergence and generalization~\cite{DP-Transformer}. Additionally, inefficient gradient clipping in shared embeddings exacerbates this loss; Abadi \textit{et al.}~\cite{DP-SGD} reported accuracy drops from 99\% to 95\% even on simple MNIST tasks.

\hl{The core challenge in privacy-preserving inertial navigation is balancing high positioning accuracy with strong privacy guarantees, a trade-off existing methods fail to optimize. Current approaches lack: (1) architectures that inherently compartmentalize gradients for efficient privacy preservation, (2) noise injection schemes that maintain the hierarchical structure of temporal features, and (3) adaptive mechanisms that align privacy perturbations with gradient utility. This gap highlights the need for frameworks that embed privacy into the architectural design itself, rather than treating it as an external constraint.}

\hl{To overcome these limitations, we propose ConvXformer, a differentially private hybrid architecture that integrates ConvNeXt blocks with Transformer encoders for robust inertial navigation. The framework introduces three key innovations: (1) a hierarchical four-stage design enabling natural gradient compartmentalization, (2) adaptive gradient clipping with truncated SVD for structured sensitivity control, and (3) Gradient-Aligned Noise Injection (GANI), which aligns perturbations with principal gradient directions to preserve utility. ConvXformer achieves strong positioning accuracy while ensuring rigorous ($\epsilon,\delta$)-differential privacy through utility-aware noise injection.}
We evaluate our approach on four benchmark datasets—RIDI~\cite{RIDI}, RoNIN~\cite{Ronin}, and OxIOD~\cite{oxiod}- and introduce Mech-IO, a new real-world dataset with magnetic interference. ConvXformer outperforms state-of-the-art methods by 35–60\% in positioning accuracy, and its privacy-preserving variant retains 15–30\% improvement under strict privacy budgets, demonstrating its effectiveness for privacy-sensitive navigation applications.

The primary contributions of this research are as follows:
\begin{enumerate}
    \item \hl{We introduce ConvXformer, a hierarchical architecture that integrates ConvNeXt blocks with Transformer encoders across four stages, resolving the fundamental CNN-Transformer computational trade-off while enabling natural gradient compartmentalization for privacy-preserving optimization.}
    \item \hl{We propose a unified compositional transformation framework combining rotation, scaling, skewing, and Gaussian noise perturbations, enhancing model generalization across diverse motion patterns and environmental conditions.}
    \item \hl{We present a privacy-preserving training framework utilizing adaptive gradient clipping and GANI, overcoming the 30-40\% performance degradation limitation of standard DP-SGD by introducing utility-weighted noise injection along dominant gradient directions.}
    \item \hl{We demonstrate superior performance across benchmark and real-world datasets, including the newly introduced Mech-IO dataset, establishing new state-of-the-art results in both accuracy and privacy preservation for inertial navigation systems.}
\end{enumerate}

The paper is organized as follows: Section II reviews data-driven inertial tracking and privacy-preserving methods. Section III outlines the theoretical foundations of differentially private inertial navigation. Section IV introduces the proposed hierarchical architecture and privacy mechanism. Section V describes the implementation methodology, while Section VI presents experimental results and ablation studies. Section VII concludes with future research directions.

\section{Related Work}
This section reviews deep learning-based inertial navigation frameworks and privacy-preserving methodologies, highlighting the distinctions between existing approaches and the proposed ConvXformer.
\subsection{Deep Learning-based Inertial Navigation Frameworks}
Deep learning architectures effectively mitigate systematic IMU integration errors, with contemporary systems dominated by three paradigms. CNN-based methods~\cite{imunet,fednav, deepils} capture local temporal patterns but fail to model long-term dependencies. RNNs handle sequential data but suffer from vanishing gradients, while Transformer-based architectures~\cite{ctin,nanomst} model global dependencies at the cost of high computational expense and susceptibility to noise. \hl{Critically, these architectures treat privacy preservation as an afterthought, applying uniform noise that disrupts the temporal hierarchy essential for inertial tracking. In contrast, ConvXformer co-designs its architecture for privacy; its four-stage hierarchy naturally compartmentalizes gradients by temporal scale, enabling structured noise injection that aligns with inertial data characteristics.}

\subsection{Differential Privacy for Location-based services}
\hl{Conventional differential privacy mechanisms face key limitations in inertial navigation: (1) uniform noise injection that ignores the temporal hierarchy of inertial features, (2) global gradient processing that introduces computational inefficiency, and (3) severe privacy–utility trade-offs leading to 30–40\% accuracy degradation.}
DP-SGD \cite{DP-SGD} \hl{applies uniform Gaussian noise across all gradient components with global clipping thresholds, achieving 95\% accuracy at ($\epsilon = 2.0$, $\delta = 10^{-5}$) compared to 99\% without privacy. However, this global approach neglects the varying sensitivity of multi-scale temporal features essential for inertial tracking. Extensions such as DP3~\cite{DP3} for indoor localization and PAPU~\cite{PAPU} for vehicular networks inherit this limitation, requiring privacy budgets too restrictive for practical deployment.}

Adaptive methods such as AdaClip~\cite{AdaCliP} improve gradient utilization via coordinate-wise adaptive clipping but incur $O(d^2)$ complexity for $d$-dimensional gradients and lack structural decomposition. Directional noise alignment by Xiang \textit{et al.} \cite{dp2} reduces privacy noise by 15–20\%, while Ben et al.~\cite{P-DP1} dynamically tune privacy parameters. Yet, these approaches remain decoupled from model architecture and overlook the temporal hierarchy inherent in inertial data, limiting their effectiveness.
\hl{ConvXformer overcomes these challenges through architectural co-design, introducing: (1) Gradient-Aligned Noise Injection (GANI) exploits natural gradient compartmentalization from the hybrid architecture, reducing computational complexity to $O(dk)$ for $k$-rank approximation; (2) Truncated SVD-based adaptive clipping that preserves principal gradient components while bounding sensitivity; and (3) Utility-weighted perturbations that align noise injection with temporal feature importance. This integration achieves superior privacy–utility trade-offs by embedding privacy within the model architecture rather than applying it as post-processing.}

\section{Preliminaries}

\subsection{Data-driven Inertial Navigation Approach}
In data-driven inertial navigation, a sequence of IMU readings \(\mathbf{X} = \{\mathbf{x}_{1}, \mathbf{x}_{2}, \dots, \mathbf{x}_{n}\}\) is used to estimate velocity. Each reading \(\mathbf{x}_t = [\mathbf{a}_t, \boldsymbol{\omega}_t] \in \mathbb{R}^6\) contains the three-axis linear acceleration \(\mathbf{a}_t\) and angular velocity \(\boldsymbol{\omega}_t\) at time step \(t\). A deep learning model learns a function \(f\) that maps these input sequences over a sliding window of \(k\) frames to a 2-D velocity vector \(\mathbf{v}_t = [v_x, v_y]\), compensating for noise, orientation drift, and sensor biases \cite{DNN-Survey}:
\begin{equation}
\mathbf{v}_t = f(\mathbf{x}_{t-k+1}, \mathbf{x}_{t-k+2}, \dots, \mathbf{x}_{t}).
\end{equation}

To reduce noise and stabilize training, the model minimizes the difference between cumulative predicted velocities (\(\mathbf{v}_i\)) and cumulative ground-truth velocities (\(\hat{\mathbf{v}}_i\)) over each sliding window:
\begin{equation}
    \text{loss} 
    = \frac{1}{m} \sum_{t=k}^n 
        \left\| 
            \sum_{i=t-k+1}^t \mathbf{v}_i \,\Delta t 
            \;-\; 
            \sum_{i=t-k+1}^t \hat{\mathbf{v}}_i \,\Delta t 
        \right\|^2,
    \label{eq:loss}
\end{equation}
where \(\|\cdot\|\) is the \(\ell_2\) norm, \(m\) is the total number of sliding windows, and \(\Delta t\) is the time interval between frames. To further mitigate drift, sensor readings are transformed into a consistent global frame using a rotation matrix \(R_t\) derived from the angular velocity sequence \(\boldsymbol{\omega}_{1:t}\).

\subsection{Differential Privacy: Definitions and Mechanisms}
Differential Privacy (DP) \cite{dp1st} is a mathematical framework that protects individual privacy by ensuring that the inclusion or exclusion of a single record in a dataset has a minimal, bounded impact on the analysis outcome. The most common variant is \((\epsilon, \delta)\)-DP \cite{DP-Survey}, which is defined as follows:

\begin{definition}
A randomized mechanism \(\mathcal{A} : \mathcal{X}^n \rightarrow \mathcal{R}\) satisfies \((\epsilon, \delta)\)-differential privacy if, for any two adjacent datasets \(D\) and \(D'\) (differing by one entry), and for any subset of outputs \(O \subseteq \mathcal{R}\), the following holds:
\begin{equation}
    \Pr[\mathcal{A}(D) \in O] \leq e^{\epsilon} \cdot \Pr[\mathcal{A}(D') \in O] + \delta.
    \label{eq:dp}
\end{equation}
\end{definition}

The privacy budget \(\epsilon\) controls the trade-off between privacy and utility, while \(\delta\) permits a small probability of breaking the stricter \(\epsilon\)-DP guarantee (where \(\delta=0\)) \cite{shokri}. DP is typically achieved by adding calibrated noise to a function's output. The amount of noise is scaled by the function's sensitivity \(\Delta f\), defined as the maximum possible change in the output when one record is changed, measured using the \(\ell_1\) norm for Laplace noise or the \(\ell_2\) norm for Gaussian noise.

\subsection{Differentially Private Stochastic Gradient Descent (DP-SGD)}
To achieve differential privacy in neural network training, Differentially Private Stochastic Gradient Descent (DP-SGD) \cite{DP-SGD} modifies the gradient update process. First, it bounds the influence of each data point by clipping the \(\ell_2\) norm of individual gradients \(g_i\) to a threshold \(C\):
\begin{equation}
    \tilde{g}_i = g_i \cdot \min\left(1, \frac{C}{\|g_i\|_2}\right).
\end{equation}
Second, it adds calibrated Gaussian noise to the averaged gradients of a mini-batch \(B\), scaled by a noise multiplier \(\sigma\):
\begin{equation}
    \bar{g} = \frac{1}{|B|} \left( \sum_{i \in B} \tilde{g}_i + \mathcal{N}\left(0, \sigma^2 C^2 \mathbf{I}\right) \right).
\end{equation}
Throughout training, a moment accountant tracks the cumulative privacy loss to ensure an overall guarantee of \((\epsilon, \delta)\)-DP.

\begin{figure} [t!]
    \centering
    \includegraphics[width=0.9\linewidth]{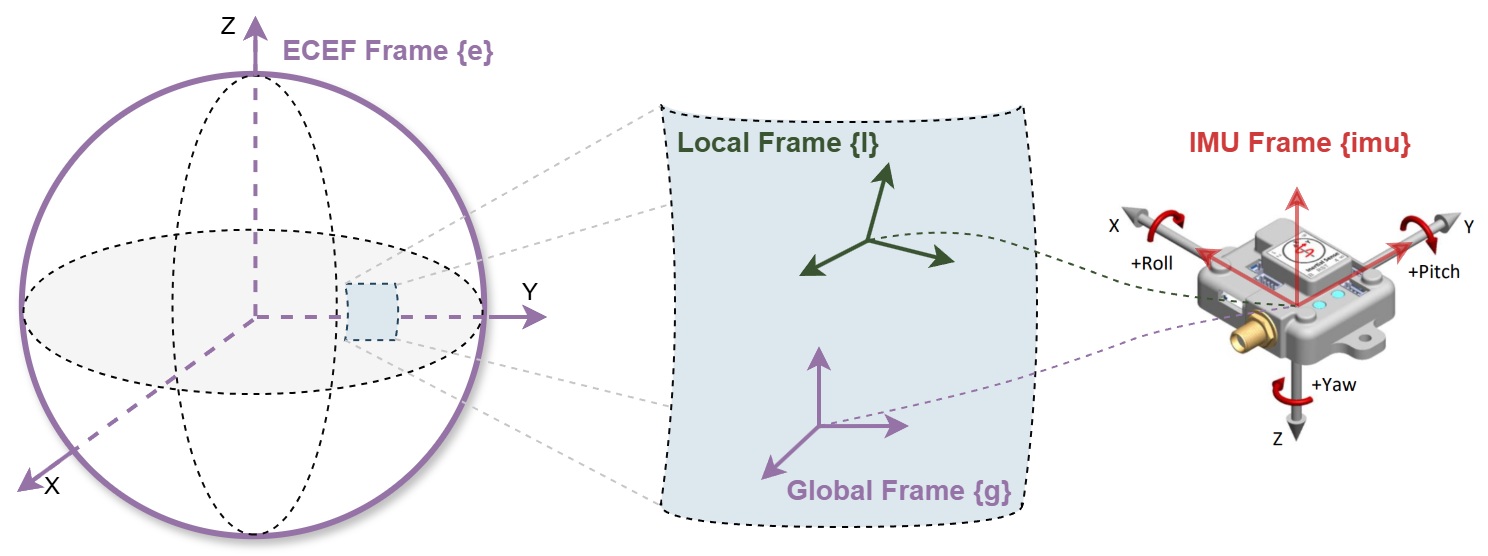}
    \caption{World and IMU coordinate frames: (a) ECEF frame \{e\} at Earth's center, with x-axis toward the prime meridian and equator, z-axis toward the North Pole; (b) ENU \{n\} and local \{l\} frames on the ground, z-axis upward; (c) IMU frame \{imu\} attached to the sensor.}
    \vspace{-0.3cm}
    \label{fig:cf}
\end{figure}

\section{Methodology}
\label{sec:methodology}
\hl{This section comprehensively explains the data-augmentation technique, proposed neural network design, and differentially private training framework}. The methodology encompasses model architecture design with ConvNeXt transformers, privacy-preserving mechanisms including adaptive gradient clipping and GANI, and theoretical foundations that underpin our approach.

\subsection{Proposed Data Augmentation Technique}
The proposed data augmentation implements a comprehensive transformation pipeline comprising rotational adjustments, scale modifications, and skew operations, culminating in the addition of Gaussian noise to accurately simulate real-world sensor behavior.
The framework initiates with a rotational transformation characterized by angle $\theta \in [0, \Theta_{\text{max}}]$, mathematically expressed as:
\begin{equation}
R = \begin{bmatrix}
\cos \theta & -\sin \theta \\
\sin \theta & \cos \theta
\end{bmatrix}
\end{equation}

Subsequently, the framework incorporates the scale variation through a multiplicative factor $s \in [1 - \Delta_s, 1 + \Delta_s]$, defined by:
\begin{equation}
S = s \cdot I
\end{equation}

To account for systematic nonlinearities, the methodology implements a skew transformation parameterized by $k \in [-\Delta_k, \Delta_k]$:
\begin{equation}
K = \begin{bmatrix}
1 & k \\
0 & 1
\end{bmatrix}
\end{equation}

These transformations are consolidated into a unified transformation matrix:
\begin{equation}
T = S \cdot (R \cdot K)
\end{equation}

The final augmented signal incorporates stochastic variation through additive Gaussian noise $\mathcal{N}(0, \sigma^2)$:
\begin{equation}
\mathbf{F}_{\text{aug}} = \mathbf{F} + \mathcal{N}(0, \sigma^2)
\end{equation}

Through extensive empirical validation, we establish optimal parameter values: $\Theta_{\text{max}} = \pi/4$ radians, $\Delta_s = 0.2$, $\Delta_k = 0.1$, and $\sigma = 0.01$. This parameter configuration achieves an optimal balance between introduced variability and preservation of essential signal characteristics.

\begin{figure} [t!]
    \centering
    \includegraphics[width=0.7\linewidth]{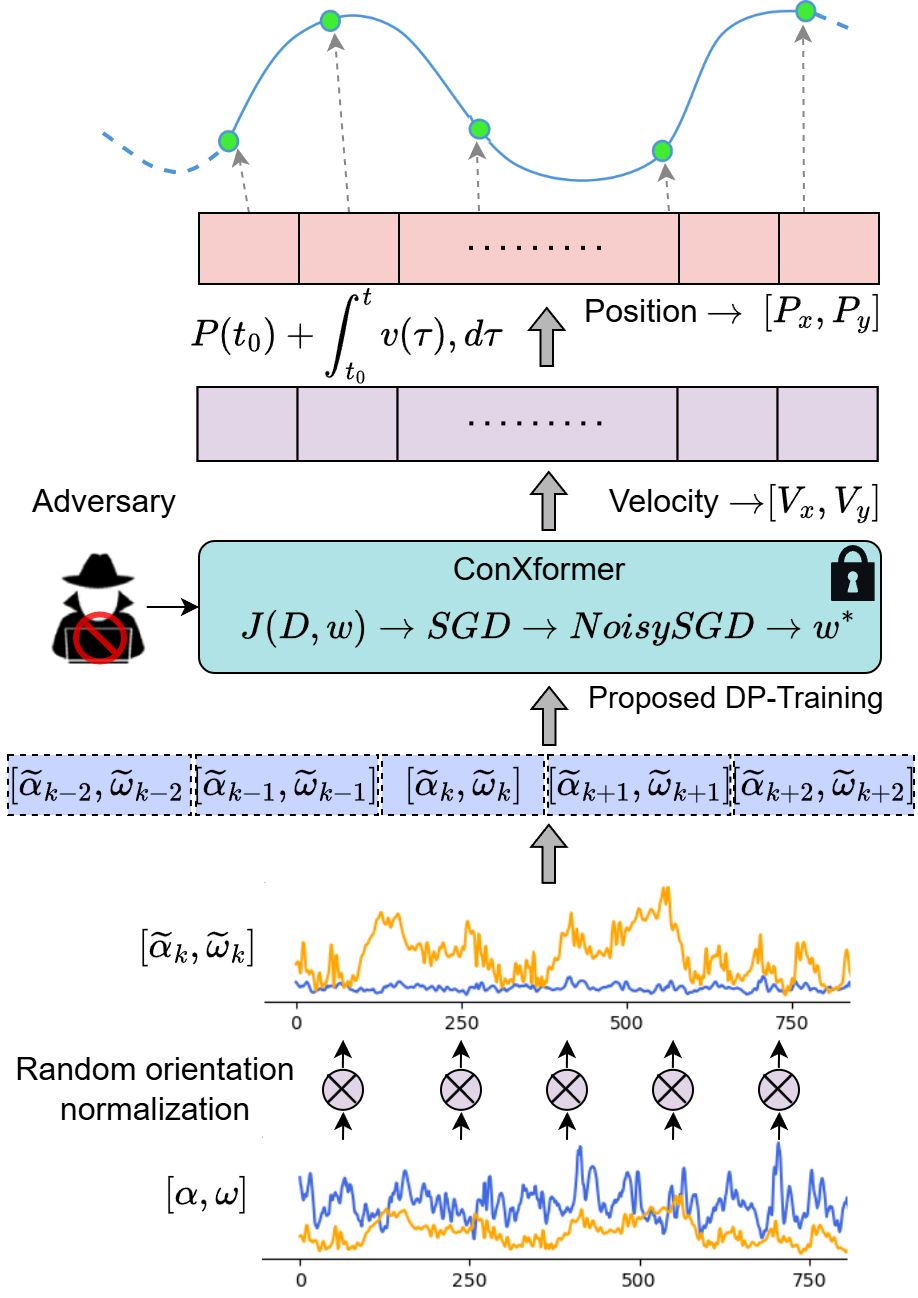}
    \caption{Differentially Private Training Framework for ConvXformer.}
    \label{fig:overview}
    \vspace{-0.3cm}
\end{figure}

\subsection{Model Architecture and Design}
This section describes our proposed ConvXformer's architecture and discusses the design strategy's significance for the inertial navigation problem, as shown in Figure \ref{fig:overview}.
\subsubsection{Overall Architecture}
The \textit{ConvXformer} is a hybrid architecture that combines the localized feature extraction of ConvNeXt-style convolutions with the global contextual modeling capabilities of Transformer-based self-attention mechanisms. By seamlessly integrating \textit{ConvNeXt blocks} and \textit{Transformer encoder layers}, ConvXformer excels at capturing both fine-grained spatial details and extended spatial dependencies, essential for inertial data processing. This multi-stage, hierarchical design with progressive downsampling balances computational efficiency and enhanced representational power.
ConvXformer adopts a four-stage, pyramid-like structure, where each stage reduces spatial resolution while increasing feature depth. This architecture aligns with recent hierarchical vision models such as ConvNeXt \cite{convnext}, combining convolutional layers for efficient local feature extraction with self-attention for broader contextual capture. The feature map resolution decreases while depth increases across stages, producing increasingly abstract representations deeper in the network.
Following a stage-wise design for computational efficiency, ConvXformer effectively captures multi-level spatial hierarchies while maintaining coherence across global features. 
ConvXformer is composed of four sequential stages with a different feature map resolution, each starting with a downsampling layer followed by several ConvNeXt blocks in each stage. Initially, a stem layer processes the input signal $X \in \mathbb{R}^{6 \times 200}$. A Conv1D with a kernel size of $k=4$ and stride $s=4$ is applied, mapping $X$ to an initial feature space of $(64, 50)$, where 64 represents the channel dimension, and 50 is the downscaled spatial dimension. Batch Normalization follows to stabilize feature distribution. Each stage has a different number of ConvNeXt blocks that integrate \textit{inverted bottleneck structures} and \textit{transformer encoder}. We present five ConvXformer variants, referred to as ConvXformer-$\alpha$, ConvXformer-$\beta$, ConvXformer-$\gamma$, and ConvXformer-$\delta$. The architectural details of these variants are stated in Table \ref{Table:model_configs}.

\subsubsection{Stage Configuration}
Balancing network depth and width under a fixed parameter budget is essential for efficient yet accurate performance. 
While ConvNeXt-T~\cite{convnext} employs a \((3,3,9,3)\) block distribution per stage, 
our baseline ConvXformer-\(\alpha\) adopts \((2,2,2,2)\), 
providing a optimal trade-off at 15\,M parameters and 180\,M FLOPs. Subsequent variants (ConvXformer-\(\beta\), \(\gamma\), and \(\delta\)) increase stage depth and channel width, raising parameter counts and FLOPs from 25\,M/270\,M to 69\,M/727\,M. These large deep models deliver improved positioning accuracy, reducing average trajectory errors by up to 0.5\,m, 
albeit at the cost of higher computational overhead.

\begin{table}[t!]
\centering
\caption{Brief Configurations of the Proposed ConvXformer}
\label{Table:model_configs}
\small
\begin{tabular}{c@{\hspace{2pt}}|@{\hspace{2pt}}c@{\hspace{2pt}}|@{\hspace{2pt}}c@{\hspace{2pt}}|@{\hspace{2pt}}c@{}|@{\hspace{2pt}}c@{}}
\hline \hline
Model & Channels & Depths & Param\hspace{2pt} & FLOP \\ \hline \hline

ConvXformer-$\alpha$ & \{64, 128, 256, 512\} & \{2, 2, 2, 2\} & 15M & 180M \\ 
ConvXformer-$\beta$ & \{64, 128, 256, 512\} & \{3, 3, 3, 3\} & 25M & 270M \\ 
ConvXformer-$\gamma$ & \{72, 144, 288, 576\} & \{3, 3, 4, 3\} & 31M & 345M \\ 
ConvXformer-$\delta$ & \{96, 192, 384, 768\} & \{4, 4, 6, 4\} & 69M & 727M \\ \hline \hline

\end{tabular}
\caption*{$\bullet$ We design four variants of the ConvXformer model with different depths=$\{B_1, B_2, B_3, B_4\}$ and channels=$\{C_1, C_2, C_3, C_4\}$.}
\vspace{-0.4cm}
\end{table}

Each downsampling layer, positioned at the beginning of each stage, applies a 1D convolution with kernel size $k=2$ and stride $s=2$. This operation halves the spatial dimension at each stage, yielding feature maps of dimensions $(128, 25)$, $(256, 12)$, $(512, 6)$, and $(512, 3)$ in subsequent stages. Batch Normalization is applied after each downsampling operation. Stage I captures fine-grained spatial details, which develop into mid-level abstractions in Stage II. Stages III and IV then progressively expand the model’s receptive field, which is essential for high-level feature representations. The channel dimension increases from $C = 64$ in the first stage to $C = 512$ in the final stage, supporting more complex feature hierarchies as spatial dimensions reduce. 

\begin{figure*}[t!]
    \centering
    \includegraphics[width=0.95\linewidth]{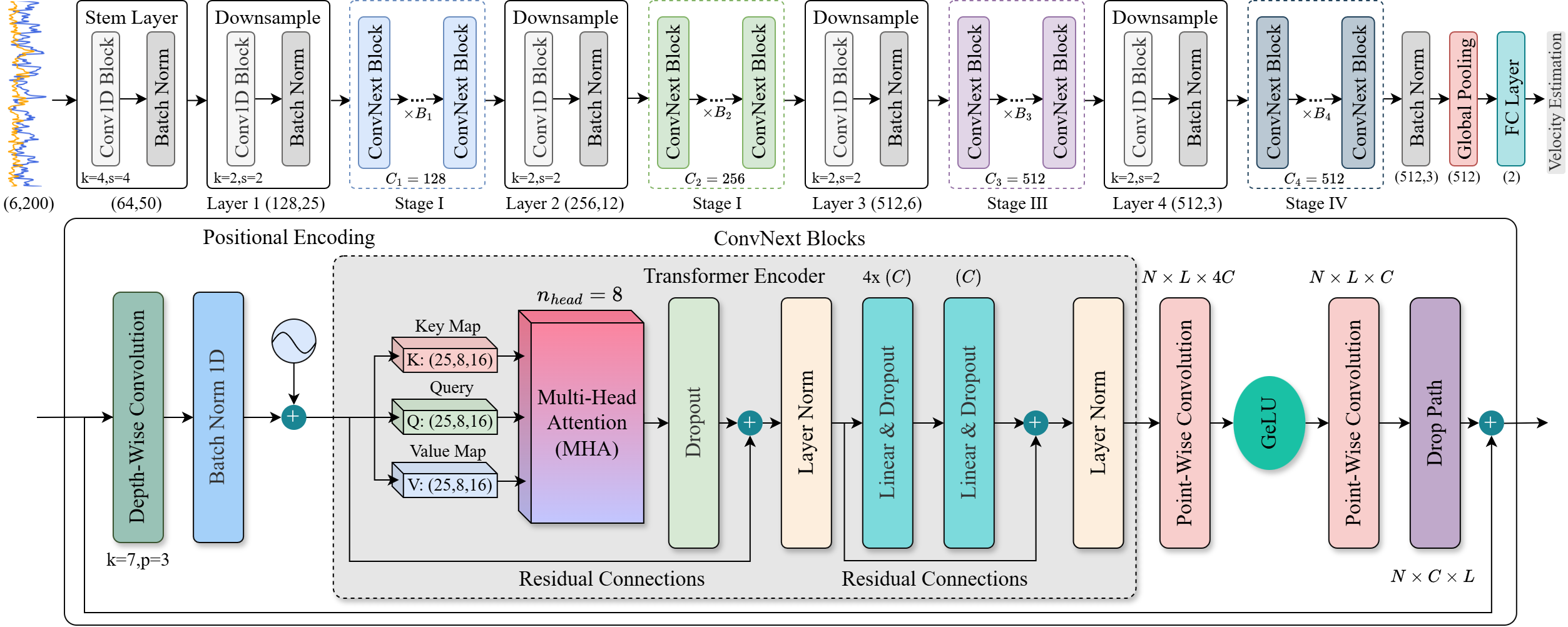}
    \caption{Overall architecture of ConvXformer. The design strategy is based on a pyramid structure with four stages similar to \cite{convnext}. Each stage comprises two ConvNeXt blocks with integrated Transformer encoders hierarchically. This figure shows the configuration of the proposed ConvXformer where $\{B_1, B_2, B_3, B_4\}$=\{2,2,2,2\}.}
    \label{fig:model}
    \vspace{-0.3cm}
\end{figure*}

\subsubsection{ConvNeXt Block with Integrated Transformer Encoder}
In each ConvNeXt block, a depthwise convolution with kernel size $k=7$, padding $p=3$, and group size equal to the channel dimension is applied. For an input feature $X \in \mathbb{R}^{C \times L}$, where $C$ is the channel dimension and $L$ the spatial dimension, the depthwise convolution outputs a feature map $Z \in \mathbb{R}^{C \times L}$ computed as:
\begin{equation}
Z = \text{DConv}_{7 \times 1}(X),
\end{equation}
which enhances localized feature interactions while preserving channel separability. This operation results in feature maps of the same dimension as the input for each block in a given stage.
ConvXformer employs an inverted bottleneck within each block to enhance representational efficiency. The structure expands the channel dimension by a factor of 4 through a pointwise convolution, followed by a GELU non-linearity and a second pointwise convolution to compress back to the original dimension $C$. If $Z \in \mathbb{R}^{C \times L}$ is the input to the inverted bottleneck, the transformations are as follows:
\begin{equation}
Y = \text{GELU}(\text{PWConv1}(Z)),\text{ where } \text{PWConv1}: \mathbb{R}^{C} \rightarrow \mathbb{R}^{4C},
\end{equation}
\begin{equation}
\text{Output} = \text{PWConv2}(Y), \text{ where } \text{PWConv2}: \mathbb{R}^{4C} \rightarrow \mathbb{R}^{C}.
\end{equation}

After the depthwise convolution, each ConvNeXt block includes a Transformer encoder with positional encoding. The positional encoding enriches the spatial coherence of the features across all stages. Transformer parameters scale across stages, from $d_{\textnormal{model}}=128$ and $d_{\textnormal{ff}}=512$ in Stage I to $d_{\textnormal{model}}=512$ and $d_{\textnormal{ff}}=2048$ in Stage IV, with $n_{\textnormal{head}}=8$ throughout.
Multi-head attention (MHA) with $n_{\textnormal{head}}=8$ captures long-range dependencies. For each input feature $X \in \mathbb{R}^{L \times C}$, the Transformer encoder generates query $Q$, key $K$, and value $V$ matrices of dimension $\mathbb{R}^{L \times C}$, and computes attention scores as:
\begin{equation}
A = \text{Softmax} \left( \frac{Q K^T}{\sqrt{d}} \right),
\end{equation}
where $d = C/n_{\textnormal{head}}$ is the scaling factor for stable gradients. 
The attended feature map is then computed as:
\begin{equation}
\text{Attention}(X) = A V.
\end{equation}
Residual connections and Layer Normalization follow each MHA and fully connected layer within the Transformer encoder block to enhance stability and performance. Unlike ConvNeXt’s Layer Normalization, our proposed ConvXformer uses Batch Normalization to stabilize training across varying batch sizes. Batch Normalization is applied after each convolutional layer, ensuring feature consistency. Additionally, DropPath regularization is employed within the Transformer encoder and at the end of each block to enhance generalization by randomly dropping paths during training.
At the end of Stage IV, a global average pooling operation aggregates spatial information, reducing the feature map $\mathbb{R}^{512 \times 3}$ to a vector of $\mathbb{R}^{512}$. A fully connected (FC) layer then maps this vector to the final output logits, aligning with the number of output classes for velocity estimation. The output logits are given by:
\begin{equation}
F_{\text{out}} = \text{FC}(\text{GAP}(F)),
\end{equation}
where $\text{GAP}(F)$ represents the globally averaged feature vector from the final stage.

\subsubsection{Micro Design}
\hl{ConvXformer's micro design is founded on three core design principles addressing fundamental inertial navigation challenges: temporal locality preservation, global dependency modeling, and computational efficiency under privacy constraints. These principles drive targeted strategies that enhance performance and reduce positioning error by 10–15\%.
ConvXformer employs a pyramid structure where ConvNeXt blocks handle local temporal patterns while Transformer encoders capture long-range dependencies. This addresses inertial data characteristics where fine-grained sensor fluctuations must integrate with broader motion patterns. The four-stage downsampling strategy (200→50→25→12→6→3) creates natural gradient compartmentalization essential for differential privacy, producing feature representations at different temporal resolutions. This contrasts with uniform architectures lacking hierarchical gradient structure, making privacy-preserving decomposition less effective.

\textit{Larger kernel than $3\times1$:} The $7 \times 1$ depthwise convolution kernel captures temporal correlations spanning 35ms windows at 200Hz sampling rate, corresponding to human gait micro-movements. Unlike vision-oriented hybrid architectures employing $3 \times 3$ kernels for 2D spatial features, this temporal-aware design enhances feature extraction and reduces drift compared to traditional $3 \times 1$ kernels.

\textit{Inverted bottleneck design:} The 4× channel expansion creates adaptive receptive fields scaling with motion complexity, enabling dynamic resource allocation. Simple motions (walking) utilize smaller receptive fields, while complex motions (running, turning) engage broader feature interactions.

\textit{Normalization and Activations:} Batch Normalization follows each depthwise and pointwise convolution, stabilizing feature distributions across batch sizes. This placement strategy, optimized for 1D temporal sequences, enhances training stability and generalization. GELU activation in each inverted bottleneck captures subtle input variations essential for precise navigation.

\textit{Regularization:} DropPath regularization introduces path-level randomness during training, promoting independent feature representations and preventing overfitting across diverse environments.
These micro-design choices enable motion-adaptive feature fusion, where the architecture dynamically adjusts computational focus based on motion complexity.}

\begin{figure*}[h!]
    \centering
    \includegraphics[width=0.95\linewidth]{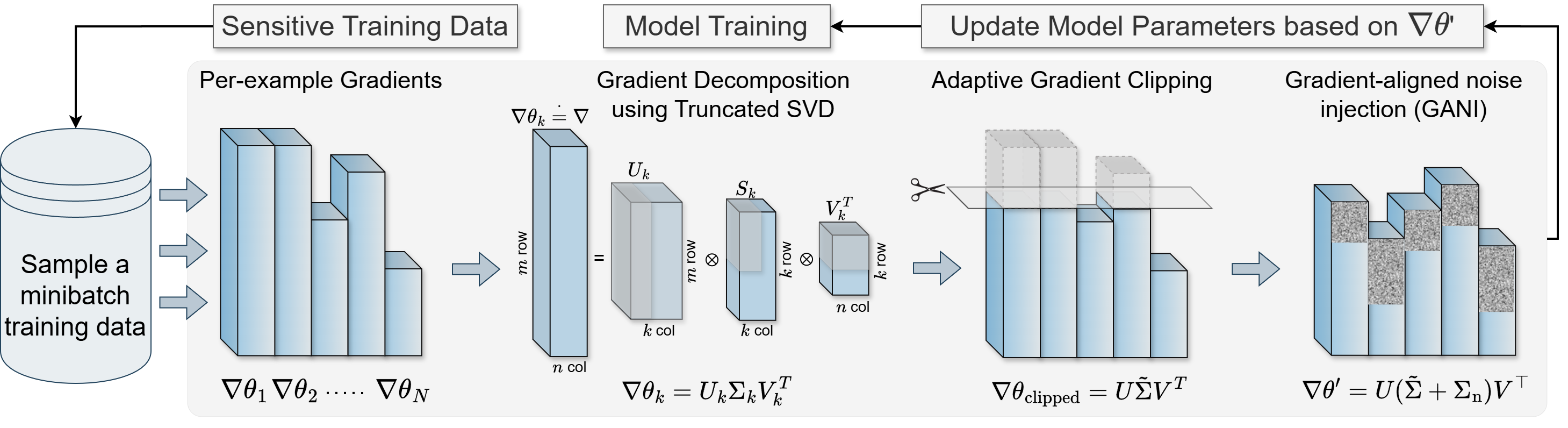}
    \caption{Overview of the proposed differential privacy mechanism. Adaptive gradient clipping dynamically adjusts thresholds based on historical norms, followed by gradient-aligned noise injection (GANI), which introduces utility-weighted Gaussian noise aligned with principal gradient directions.}
    \label{fig:dp_method}
    \vspace{-0.2cm}
\end{figure*}

\subsection{Proposed Differential Privacy Mechanism}
Our differential privacy framework preserves gradient fidelity by combining adaptive clipping, truncated Singular Value Decomposition (SVD), and a novel utility-aware, directionally aligned noise injection scheme (Algorithm~\ref{alg:proposed_dp}). The core innovation is scaling noise relative to singular values, which emphasizes dominant gradient directions and enhances the privacy-utility trade-off.
Under our threat model, an adversary can exploit intermediate gradients to perform membership inference or location-recovery attacks and reconstitute individual trajectories \cite{DP-Survey}. To counter this, we adopt a formal $(\epsilon,\delta)$-differential privacy mechanism that bounds the influence of any single sample on the training process, mitigating both external and insider threats more effectively than traditional methods.

The fundamental insight is to decompose each gradient tensor $\nabla\theta$ into a low-rank approximation using truncated SVD. This isolates the principal gradient components while attenuating less significant, noisy signals. The resulting rank-$k$ truncated gradient tensor, $\nabla\theta_k$, is formally expressed as:
\begin{equation}
\label{eq:1}
\nabla\theta_k = \mathbf{U}_k \mathbf{\Sigma}_k \mathbf{V}_k^\top,
\end{equation}
where $\mathbf{U}_k \in \mathbb{R}^{m \times k}$ and $\mathbf{V}_k \in \mathbb{R}^{n \times k}$ are the matrices of the top $k$ left and right singular vectors, and $\mathbf{\Sigma}_k \in \mathbb{R}^{k \times k}$ is the diagonal matrix of their corresponding singular values. For efficient computation, the gradient tensor is matricized before this decomposition.

\begin{algorithm}[h!]
\caption{SVD-based Gradient Processing and GANI}
\label{alg:proposed_dp}

\KwIn{Training data $D_{\text{train}}$, validation data $D_{\text{val}}$, batch size $B$, learning rate $\eta$, epochs $T$, initial threshold $\lambda_0$, noise scale $\sigma$, privacy params $(\epsilon,\delta)$, truncation rank $k$, momentum $\gamma$}
\KwOut{Model $\theta_T$, privacy cost $\epsilon_{\text{values}}$, losses $\mathcal{L}_{\text{train}}, \mathcal{L}_{\text{val}}$}

Initialize $\theta_0$ and the privacy accountant with $(\epsilon,\delta)$\;
$\lambda_{\text{hist}} \gets \lambda_0$\;

\For{epoch $e \gets 1$ \KwTo $T$}{
  \For{mini-batch $(X_t,y_t) \sim D_{\text{train}}$}{
    Compute loss $\mathcal{L}_{\text{total}}$ and gradients $\nabla\theta$\;

    \For{each parameter tensor $p$}{
      \tcp{SVD decomposition}
      $\mathbf{U}_k, \mathbf{\Sigma}_k, \mathbf{V}_k \gets \mathrm{SVD}(\nabla\theta, k)$\;

      \tcp{Adaptive clipping}
      $\lambda_{\text{hist}} \gets \gamma \lambda_{\text{hist}} + (1-\gamma)\|\nabla\theta\|_2$\;
      Clip singular values using Eq.~\ref{eq:3}\;

      \tcp{Noise generation with SVD alignment}
      $w_i \gets {\sigma_i}/{\sum_{j=1}^k \sigma_j}$\;
      Draw $\xi \in \mathbb{R}^{k \times k} \sim \mathcal{N}(0, \sigma \lambda_{\text{hist}})$\;
      $\mathbf{\Sigma}_{\text{n}} \gets \xi \odot w$\;
      $\mathbf{z} \gets \mathbf{U}_k \mathbf{\Sigma}_{\text{n}} \mathbf{V}_k^\top$\;

      \tcp{Final DP update}
      Compute DP gradients using Eq.~\ref{eq:7}\;
    }

    $\theta_{t+1} \gets \theta_t - \eta\,\nabla\theta'$\;
    $\epsilon \gets \mathrm{RDP}\!\left(\sigma, \frac{B}{|D_{\text{train}}|}, e, \delta\right)$\;
  }

  \If{$D_{\text{val}}$ is provided}{
    Compute $\mathcal{L}_{\text{val}}$ and update scheduler\;
  }
}

\KwRet{$\theta_T,\ \epsilon_{\text{values}},\ \mathcal{L}_{\text{train}},\ \mathcal{L}_{\text{val}}$}

\end{algorithm}

\subsubsection{Adaptive Gradient Clipping}
To establish bounded sensitivity, we implement an adaptive clipping mechanism on the singular values $\sigma_i$, with dynamically adjusted thresholds at each optimization step. Using a momentum coefficient $\gamma \in (0, 1)$, the threshold $\lambda_t$ at iteration $t$ is computed as:

\begin{equation}
\label{eq:2}
\lambda_t = \eta \lambda_{t-1} + (1 - \eta) \|\nabla\theta\|_2,
\end{equation}

where $\|\nabla\theta\|_2$ denotes the Euclidean norm of the gradient tensor, $\lambda_t$ represents the current threshold, and $\eta \in (0, 1)$ is a momentum coefficient. The clipping is applied to singular values through Equation~\ref{eq:3}:

\begin{equation}
\label{eq:3}
\tilde{\sigma}_i = \min(\sigma_i, \lambda_t).
\end{equation}

This procedure constrains the influence of dominant singular values while preserving essential gradient information. As shown in Equation~\ref{eq:4}, the clipped gradient tensor is reconstructed as:

\begin{equation}
\label{eq:4}
\nabla\theta_{\text{clipped}} = \mathbf{U}_k \tilde{\mathbf{\Sigma}}_k \mathbf{V}_k^\top,
\end{equation}

where $\tilde{\mathbf{\Sigma}}_k$ denotes a diagonal matrix constructed from the clipped singular values $\{\tilde{\sigma}_1, \tilde{\sigma}_2, \ldots, \tilde{\sigma}_k\}$.

\subsubsection{Gradient-aligned noise injection (GANI)}
We introduce a novel structured perturbation mechanism that aligns Gaussian noise with gradient utility. For each parameter tensor, we generate isotropic noise $\xi \in \mathbb{R}^{k \times k} \sim \mathcal{N}(0, \sigma\lambda_t)$ and align it with the gradient's principal components. The perturbation tensor is generated as:

\begin{equation}
\label{eq:5}
\mathbf{z} = \mathbf{U}_k \mathbf{\Sigma}_{\text{n}} \mathbf{V}_k^\top,
\end{equation}

where $\mathbf{\Sigma}_{\text{n}}$ represents a diagonal matrix of calibrated noise intensities. These intensities are derived from normalized singular values as shown in Equation~\ref{eq:6}:

\begin{equation}
\label{eq:6}
w_i = \frac{\sigma_i}{\sum_{j=1}^k \sigma_j},
\end{equation}

ensuring noise concentration along high-utility gradient directions through $\mathbf{\Sigma}_{\text{n}} = \xi \odot w$. The final DP-enforced gradient update combines the clipped gradients with aligned noise:

\begin{equation}
\label{eq:7}
\nabla\theta' = \mathbf{U}_k (\tilde{\mathbf{\Sigma}}_k + \mathbf{\Sigma}_{\text{n}}) \mathbf{V}_k^\top.
\end{equation}
This formulation maintains gradient informativeness while providing rigorous $(\epsilon, \delta)$-differential privacy guarantees.
\subsubsection{Privacy Analysis}
We analyze the privacy guarantees using Rényi Differential Privacy (RDP). At each iteration, for sampling rate $\gamma = |B|/N$ where $|B|$ is batch size, and $N$ is dataset size, we compute the RDP loss with order $\lambda > 1$ as:

\begin{equation}
\mathcal{R}(\lambda) = \frac{1}{\lambda - 1} \log\left(1 + \gamma^2\lambda\binom{2\lambda}{\lambda}\left(e^{\lambda/\sigma^2} - 1\right)\right),
\end{equation}

where $\sigma$ is the noise multiplier and the binomial coefficient captures privacy amplification through subsampling. Over $T$ training iterations, the total RDP accumulates additively:

\begin{equation}
\mathcal{R}_{\text{total}}(\lambda) = T \cdot \mathcal{R}(\lambda).
\end{equation}

The final $(\epsilon, \delta)$-DP guarantee is obtained by optimizing over a fine-grained sequence of RDP orders $\lambda$, generated as $\{1.1 + 0.1i\}_{i=1}^{300}$:

\begin{equation}
\epsilon = \min_{\lambda > 1} \left[\mathcal{R}_{\text{total}}(\lambda) - \frac{\log(\delta)}{\lambda - 1}\right].
\end{equation}

\section{Systematic and Theoretical Analysis}
We establish our privacy guarantees through a systematic analysis of the proposed mechanism. Our analysis examines three fundamental aspects: sensitivity bounds for the truncated SVD mechanism, composition properties under gradient-aligned noise, and the overall privacy guarantees under iteration.

\begin{theorem}[SVD Sensitivity]\label{thm:svd_sens}
For the truncated SVD mechanism with adaptive clipping threshold $\lambda_t$, rank $k$, and adjacent datasets $\mathcal{D}$ and $\mathcal{D}'$, the $\ell_2$-sensitivity satisfies:
\begin{equation}\label{eq:sensitivity}
    \Delta_2(\nabla\theta_k) \leq 2\lambda_t\sqrt{k}
\end{equation}
where $\nabla\theta_k$ and $\nabla\theta'_k$ are the truncated gradient tensors computed on $\mathcal{D}$ and $\mathcal{D}'$ respectively.
\end{theorem}

\begin{proof}
Consider gradient tensors $\mathbf{G}$ and $\mathbf{G}'$ on adjacent datasets with truncated SVD decompositions:
\begin{align}\label{eq:svd_decomp}
    \nabla\theta_k &= \mathbf{U}_k\mathbf{\Sigma}_k\mathbf{V}_k^\top \\
    \nabla\theta'{_k} &= \mathbf{U}'_k\mathbf{\Sigma}'_k{\mathbf{V}'_k}^\top \notag
\end{align}

Using orthonormality of $\mathbf{U}_k$ and $\mathbf{V}_k$, combined with our clipping of singular values:
\begin{align}\label{eq:sens_bound}
    \|\nabla_k - \nabla'_k\|_F^2 &= \|\mathbf{U}_k\tilde{\mathbf{\Sigma}}_k\mathbf{V}_k^\top - \mathbf{U}'_k\tilde{\mathbf{\Sigma}}'_k{\mathbf{V}'_k}^\top\|_F^2 \notag\\
    &\leq \sum_{i=1}^k (\tilde{\sigma}_i - \tilde{\sigma}'_i)^2 \leq 4\lambda_t^2k
\end{align}

This final bound emerges from the triangle inequality and our clipping mechanism, ensuring each clipped singular value differs by at most $2\lambda_t$ between adjacent datasets.
\end{proof}

\begin{lemma}[Gradient-Aligned Noise Composition]\label{lem:noise_comp}
Building on the sensitivity bound from Theorem~\ref{thm:svd_sens}, for $T$ iterations with gradient-aligned noise injection, the RDP guarantee satisfies:
\begin{equation}\label{eq:rdp_total}
    \epsilon_{\text{total}}(\alpha) = \frac{T\alpha}{2\sigma^2}\left(\sum_{i=1}^k w_i^2\right)
\end{equation}
where $w_i$ are the normalized utility weights and $\sigma$ is the noise scale.
\end{lemma}

\begin{proof}
Analyzing the privacy loss at iteration $t$, denoted by $Z_t$, our gradient-aligned noise mechanism yields:
\begin{equation}\label{eq:noise_norm}
    \|\mathbf{z}\|_2^2 = \sum_{i=1}^k w_i^2\|\mathbf{n}_i\|_2^2
\end{equation}
where $\mathbf{n}_i$ represents independent Gaussian vectors.

The moment generating function of the privacy loss $Z_t$ takes the form:
\begin{equation}\label{eq:mgf}
    \mathbb{E}[e^{\alpha Z_t}] = \exp\left(\frac{\alpha}{2\sigma^2}\sum_{i=1}^k w_i^2\right)
\end{equation}

Through the additivity property of RDP under composition across $T$ iterations, we obtain equation~\eqref{eq:rdp_total}.
\end{proof}

\begin{theorem}[Overall Privacy Guarantee]\label{thm:privacy}
Using the RDP bounds from Lemma~\ref{lem:noise_comp}, the mechanism achieves $(\epsilon,\delta)$-differential privacy where:
\begin{equation}\label{eq:dp_guarantee}
    \epsilon = \min_{\alpha > 1} \left[\epsilon_{\text{total}}(\alpha) + \frac{\log(1/\delta)}{\alpha-1}\right]
\end{equation}
with probability at least $1-\delta$.
\end{theorem}

\begin{proof}
Starting with $\epsilon_{\text{total}}(\alpha)$ from Lemma~\ref{lem:noise_comp} and applying the conversion theorem from RDP to $(\epsilon,\delta)$-DP, we can state that for any $\alpha > 1$ and $\delta > 0$, our RDP guarantee implies $(\epsilon_{\text{total}}(\alpha) + \frac{\log(1/\delta)}{\alpha-1}, \delta)$-DP.

The optimal privacy parameter $\epsilon$ is achieved through minimization over $\alpha$ as shown in equation~\eqref{eq:dp_guarantee}. This minimum exists due to the convexity of our expression: as $\alpha$ approaches 1, the logarithmic term grows unboundedly, while for $\alpha$ approaching infinity, the first term becomes dominant.
\end{proof}

\begin{lemma}[Privacy Amplification]\label{lem:amplification}
Extending Theorem~\ref{thm:privacy} with sampling rate $q$, the effective RDP parameter satisfies:
\begin{equation}\label{eq:amplification}
    \tilde{\epsilon}(\alpha) \leq \frac{q^2\alpha}{2\sigma^2}\left(\sum_{i=1}^k w_i^2\right)
\end{equation}
where $\alpha > 1$ is the order of RDP.
\end{lemma}

\begin{proof}
The privacy loss under sampling becomes:
\begin{equation}\label{eq:sampled_loss}
    Z_t^q = \begin{cases}
        Z_t & \text{with probability } q \\
        0 & \text{with probability } 1-q
    \end{cases}
\end{equation}

For our gradient-aligned noise mechanism, the moment generating function of $Z_t$ conditioned on the sampling event $S$ is:
\begin{align}\label{eq:sampled_mgf}
    \mathbb{E}[e^{\alpha Z_t}|S] &= \mathbb{E}\left[\exp\left(\alpha\log\frac{p(\tilde{\nabla}\theta_t|\nabla\theta_t)}{p(\tilde{\nabla}\theta_t|\nabla\theta'_t)}\right)\right] \notag\\
    &= \exp\left(\frac{\alpha\|\nabla\theta_t - \nabla\theta'_t\|_2^2}{2\sigma^2}\sum_{i=1}^k w_i^2\right)
\end{align}

The bound is tight due to gradient clipping ensuring $\|\nabla\theta\|_2 \leq 1$, while random sampling reduces effective sensitivity by factor $q$.
\end{proof}

\begin{lemma}[Convergence Analysis]\label{lem:convergence}
Let $f(\theta)$ be $L$-smooth and $\mu$-strongly convex. Using the noise mechanism from Lemma~\ref{lem:noise_comp} with learning rate $\eta_t \leq \frac{1}{L}$, for iterations $t = 1,\ldots,T$, the expected optimization error satisfies:
\begin{equation}\label{eq:convergence}
    \mathbb{E}[f(\theta_T) - f(\theta^*)] \leq (1-\mu\eta_t)^T[f(\theta_0) - f(\theta^*)] + \frac{\eta_t\sigma^2}{2\mu}\sum_{i=1}^k w_i^2
\end{equation}
where $\theta^*$ is the optimal parameter value and $w_i$ are the utility weights from equation~\eqref{eq:rdp_total}.
\end{lemma}

\begin{proof}
At iteration $t$, we update parameters as:
\begin{equation}\label{eq:update_rule}
    \theta_{t+1} = \theta_t - \eta_t(\nabla\theta_t + \mathbf{z}_t)
\end{equation}

By $L$-smoothness and $\mu$-strong convexity:
\begin{align}\label{eq:smooth_convex}
    f(\theta_{t+1}) &\leq f(\theta_t) - \eta_t\langle\nabla f(\theta_t), \tilde{\nabla}\theta_t\rangle \notag\\
    &\quad + \frac{L\eta_t^2}{2}\|\tilde{\nabla}\theta_t\|^2
\end{align}
\begin{equation}\label{eq:strong_convex}
    \langle\nabla f(\theta_t), \theta_t - \theta^*\rangle \geq f(\theta_t) - f(\theta^*) + \frac{\mu}{2}\|\theta_t - \theta^*\|^2
\end{equation}

Taking expectation over $\mathbf{z}_t$ and using $\eta_t \leq \frac{1}{L}$:
\begin{equation}\label{eq:iteration_bound}
    \mathbb{E}[f(\theta_{t+1}) - f(\theta^*)] \leq (1-\mu\eta_t)\Delta_t + \frac{\eta_t\sigma^2}{2}\sum_{i=1}^k w_i^2
\end{equation}
where $\Delta_t = f(\theta_t) - f(\theta^*)$. Applying recursively:
\begin{equation}\label{eq:recursive_bound}
    \Delta_T \leq (1-\mu\eta_t)^T\Delta_0 + \frac{\eta_t\sigma^2}{2\mu}\sum_{i=1}^k w_i^2
\end{equation}

The bound shows exponential decay of initial error at rate $(1-\mu\eta_t)^T$, with an additive term from privacy noise proportional to $\frac{\eta_t\sigma^2}{2\mu}\sum_{i=1}^k w_i^2$ from Lemma~\ref{lem:noise_comp}.
\end{proof}

Our analysis demonstrates a comprehensive theoretical framework where Theorem~\ref{thm:svd_sens} establishes fundamental sensitivity bounds, extended through Lemma~\ref{lem:noise_comp} to analyze privacy composition. Theorem~\ref{thm:privacy} converts these guarantees to $(\epsilon,\delta)$-DP, while Lemma~\ref{lem:amplification} provides additional privacy through subsampling. Finally, Lemma~\ref{lem:convergence} establishes utility bounds using the exact noise mechanism, showing that our approach maintains linear convergence while providing strong privacy guarantees.

\section{Performance Evaluation and Discussion}
This section presents a comprehensive evaluation of our proposed framework, starting with a primer on a framework's pipeline and its integration with our differential privacy mechanisms, followed by extensive performance analysis on benchmark and real-world datasets.

\subsection{Pipeline Architecture and Data Preprocessing}
Our end-to-end pipeline consists of four main stages: data preprocessing, compositional data augmentation, ConvXformer feature extraction, and differentially private training. The process begins by transforming raw IMU measurements, comprising gyroscopic angular rates ($\omega$), linear accelerations ($\alpha$), and ARCore orientation quaternions, into clean, globally aligned features. We resample all inertial sequences to 200 Hz and apply calibration procedures to mitigate sensor biases and noise.
To ensure orientation consistency across all stages of our hierarchical model, we transform IMU measurements from the local device frame to a global reference frame (Fig.~\ref{fig:cf}). Angular velocities and linear accelerations are first aligned to the East-North-Up (ENU) Local Frame and then mapped to the Global Frame for geographic consistency. This transformation is performed using the orientation quaternion $\mathbf{q}_{\text{ori}}$:
\begin{equation}
    \mathbf{q}_{\text{glob}} = \mathbf{q}_{\text{ori}} \cdot \mathbf{q}_{\text{imu}} \cdot \mathbf{q}_{\text{ori}}^*,
\end{equation}
where $\mathbf{q}_{\text{ori}}^*$ is the quaternion conjugate. The resulting globally aligned vectors, $\tilde{\alpha}_k$ and $\tilde{\omega}_k$, are normalized to eliminate device orientation and heading variations. Ground-truth target velocities are then computed from high-fidelity position references (e.g., ARCore) using finite differences:
\begin{equation}
    V_{x,y}(t) = \frac{\text{pos}_{x,y}(t + \Delta t) - \text{pos}_{x,y}(t)}{\Delta t},
\end{equation}
where $\Delta t$ is the resampling interval. The final processed features $(\tilde{\alpha}, \tilde{\omega}) \in \mathbb{R}^{6 \times 200}$ serve as orientation-normalized inputs to the ConvXformer stem layer. These sequences then undergo data augmentation before being passed to our differentially private training mechanism.

\subsection{Evaluation Metrics}
We evaluate the proposed model using four metrics that assess different aspects of trajectory estimation quality: Absolute Trajectory Error (ATE), Relative Trajectory Error (RTE), Scale Consistency (SC), and Cumulative Distribution Function (CDF).
\subsubsection{Absolute Trajectory Error (ATE)}
ATE measures the global positional drift between estimated and ground truth trajectories:
\begin{equation}
\text{ATE} = \sqrt{\frac{1}{N}\sum_{i=0}^{N}|\mathbf{p}{\text{gt},i}-\mathbf{p}{i}|2^2},
\label{eq:ate}
\end{equation}
where $\mathbf{p}{\text{gt},i}$ and $\mathbf{p}_{i}$ are the ground-truth and estimated positions at time step $i$, and $N$ is the total samples.
\subsubsection{Relative Trajectory Error (RTE)}
RTE evaluates local consistency over interval $\Delta t$:
\begin{equation}
\sqrt{\frac{1}{N-\Delta t}\sum_{i=0}^{N-\Delta t}\left| (\mathbf{p}{\text{gt},i+\Delta t}-\mathbf{p}{\text{gt},i}) - (\mathbf{p}{i+\Delta t}-\mathbf{p}{i}) \right|_2^2},
\label{eq:rte}
\end{equation}
comparing ground-truth and estimated displacements to assess temporal coherence.
\subsubsection{Scale Consistency (SC)}
SC quantifies trajectory scale stability across temporal windows:
\begin{equation}
s_i = \frac{|\mathbf{p}{i+\Delta w}-\mathbf{p}{i}|}{|\mathbf{p}{\text{gt},i+\Delta w}-\mathbf{p}{\text{gt},i}|}, \quad
\text{SC} = \sqrt{\frac{1}{M}\sum_{i=0}^{M}(s_i-\bar{s})^2},
\label{eq:sc}
\end{equation}
where $s_i$ is the scale ratio over window $\Delta w$, $M$ is the total windows, and $\bar{s}$ is the mean ratio. Scale drift represents $s_i$ temporal variability.
\subsubsection{Cumulative Distribution Function (CDF)} CDF provides a statistical characterization of error magnitudes, illustrating the proportion of errors below a given threshold:
\begin{equation}
   F(e) = P(E \le e) = \int_0^e f(x)\,dx,
   \label{eq:cdf}
\end{equation}
where $f(x)$ represents the probability density function of the error, $e_i = \|\mathbf{p}_{\text{gt},i} - \mathbf{p}_{i}\|_2$ defines the error magnitude for the $i$-th time step, and $e_{\max} = \max_i e_i$ is the maximum observed error. The empirical CDF is derived by sorting all error magnitudes and computing the fraction below each threshold $e$, enabling systematic comparison across different models.

\begin{figure}[t!]
\vspace{-0.1cm}
    \centering
    \includegraphics[width=0.9\linewidth]{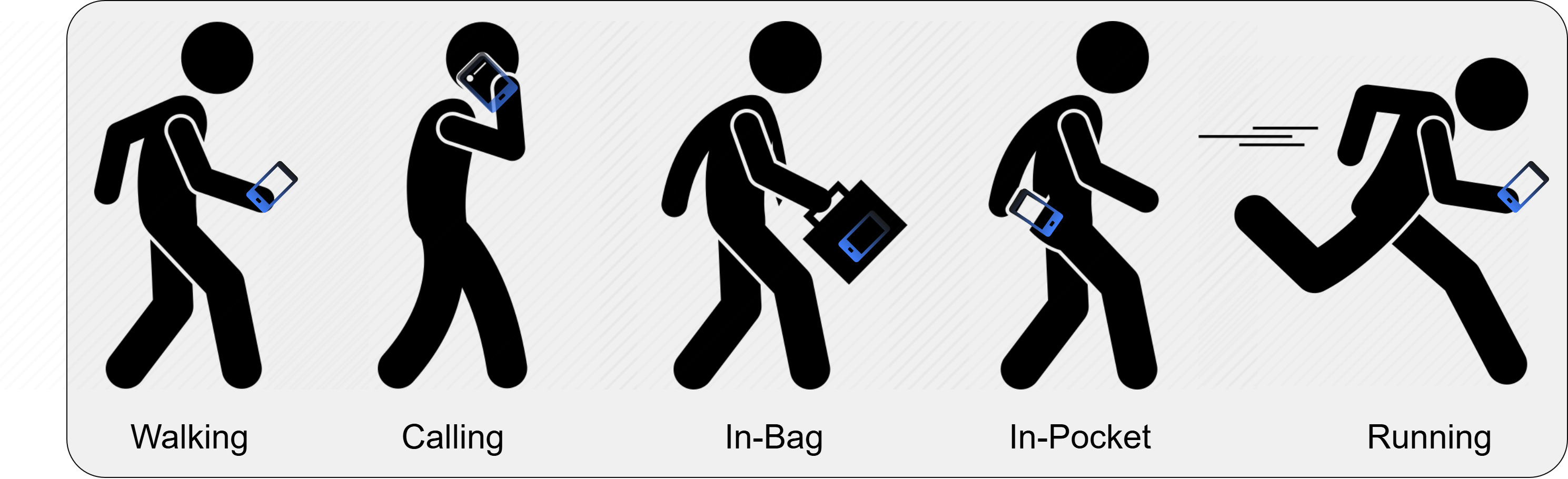}
    \caption{Smartphone IMU data acquisition, three participants carried their smartphones in five positions: walking, talking, bag, inside bag, and pocket.}
    \label{fig:mechio-dataset}
    \vspace{-0.3cm}
\end{figure}

\subsection{Experimental Setup and Implementation}
\hl{We implement our framework in PyTorch and conduct experiments on NVIDIA RTX 4090 GPUs. The architecture extends ConvXformer-$\alpha$ with novel differential privacy training mechanisms that advance privacy-preserving optimization through innovative gradient processing methodology. The ConvXformer architecture employs hierarchical feature extraction with depths $[2, 2, 2, 2]$ and channel dimensions $[64, 128, 256, 512]$ across four stages. Each ConvNeXtBlock integrates multi-head self-attention with 8 attention heads, positional encoding, and depthwise convolutions with a kernel size of 7. The input processing utilizes 6-channel IMU sequences of length 200 with a stride of 10, generating 2-dimensional velocity predictions.
Training configuration employs Stochastic Gradient Descent with momentum $\beta = 0.9$ and initial learning rate $\eta = 10^{-3}$ over 200 epochs with batch size 128. The objective function combines Mean Squared Error $\mathcal{L}_{\text{MSE}}$ with differential loss $\mathcal{L}_{\text{diff}} = \mathbb{E}[|\hat{v} - v|]$ for enhanced trajectory consistency. Learning rate scheduling utilizes ReduceLROnPlateau with reduction factor $\gamma = 0.1$ and patience threshold $\tau = 10$ epochs. Regularization incorporates drop path rate 0.1 and random horizontal rotation augmentation during training.

Our training approach employs a three-stage privacy-preserving pipeline: (1) gradient decomposition via truncated Singular Value Decomposition (SVD), (2) adaptive gradient clipping with historical norm adaptation factor $\alpha = 0.9$, and (3) gradient-aligned noise generation based on decomposed components. The privacy configuration utilizes noise multiplier $\sigma = 2.0$, gradient norm threshold $C = 2.0$, and privacy budget $\delta = 10^{-5}$.
The model trains for 200 epochs using a composite loss function combining Mean Squared Error ($\mathcal{L}_{\text{MSE}}$) with a differential component ($\mathcal{L}_{\text{diff}}$) to optimize pedestrian velocity prediction accuracy. We employ Stochastic Gradient Descent with momentum $\beta = 0.9$ and initial learning rate $\eta = 10^{-3}$. A ReduceLROnPlateau scheduler adapts the learning rate with patience threshold $\tau = 10$ epochs and reduction coefficient $\gamma = 0.1$ to ensure optimal convergence.
Privacy accounting employs a specialized framework monitoring R\'{e}nyi Differential Privacy (RDP) throughout training, converting accumulated RDP values to $(\epsilon, \delta)$-DP guarantees. This approach enables precise privacy loss quantification while maintaining utility through decomposition-based noise generation. The implementation optimizes computational efficiency through batched tensor operations, with truncated SVD providing effective dimensionality reduction while preserving essential gradient information.}

\begin{table*}[t!]
\centering
\caption{Performance Evaluation of ConvXformer and ConvXformer-DP with S.O.T.A. Inertial Tracking frameworks on Proposed and Benchmark Datasets.}
\label{table:benchmark}
\resizebox{\textwidth}{!}{
\renewcommand{\arraystretch}{1.1}
\begin{tabular}{l|c|c|c|c|c|c|c|c|c|c}
\hline \hline
\multirow{2}{*}{Datasets} & \multirow{2}{*}{Subjects} & \multirow{2}{*}{Metric} & \multirow{2}{*}{PDR} & \multirow{2}{*}{IONet} & \multirow{2}{*}{CTIN} & \multirow{2}{*}{IMUNet} & \multirow{2}{*}{ResMixer} & \multirow{2}{*}{LLIO} & \multirow{2}{*}{ConvXformer$^*$} & \multirow{2}{*}{ConvXformer-DP$^*$} \\
 &   &  &  &  &  & & & & & \\ \hline \hline
 
\multirow{2}{*}{RIDI} & \multirow{2}{*}{Test data} & ATE & 22.76 & 3.25 & 1.86 & 1.56 & 3.08 & 3.13 & 1.43 & 1.81\\ 
\cline{4-11}
& & RTE & 24.89 & 2.64 & 2.49 & 1.83 & 2.64 & 2.78 & 1.73 & 1.70\\ \hline

\multirow{4}{*}{RoNIN} & \multirow{2}{*}{Seen data} & ATE & 26.64 & 6.18 & 4.62 & 3.52 & 3.77 & 4.67 & 3.21 & 3.42 \\ 
\cline{4-11}
 &  & RTE & 23.82 & 4.94 & 2.81 & 2.85 & 2.71 & 3.57 & 2.64 & 2.66 \\ \cline{4-11}
\cline{3-11}
& \multirow{2}{*}{Unseen data} & ATE & 23.49 & 8.16 & 5.61 & 5.68 & 5.25 & 5.42 & 4.63 & 4.74\\ \cline{4-11}
&  & RTE & 23.07 & 6.37 & 4.48 & 4.49 & 4.55 & 4.81 & 3.06 & 3.47\\ \cline{4-11}
\cline{1-11}

\multirow{2}{*}{OxIOD} & \multirow{2}{*}{Test data} & ATE & 17.72 & 2.20 & 3.34 & 2.88 & 1.72 & 2.21 & 1.63 & 2.25\\
\cline{4-11}
 &   & RTE & 17.21 & 2.42 & 1.33 & 2.58 & 2.11 & 2.16 & 1.34 & 2.19 \\ \hline

\multirow{2}{*}{Mech-IO} & \multirow{2}{*}{Test data} & ATE & 29.97 & 3.36 & 2.99 & 3.04 & 2.76 & 3.63 & 1.27 & 3.10\\
\cline{4-11}
 &  & RTE & 29.15 & 4.09 & 3.41 & 3.56 & 3.18 & 3.37 & 1.39 & 2.98\\ \hline \hline

\end{tabular}
}

\footnotesize{\vspace{0.1cm}\hspace{-15.7cm}Proposed $(*)$}
\vspace{-0.3cm}
\end{table*}

\subsection{Proposed Dataset Mech-IO:}
We present Mech-IO, a comprehensive inertial measurement dataset collected in the Mechanical Engineering building at KAIST's Daejeon campus. The data acquisition spans four floors, capturing challenging environments with significant magnetic interference from nearby high-voltage machinery and equipment. The ground-truth trajectories were collected using three Android smartphones running an Android application that integrates the ARCore API. This application leverages visual-inertial odometry techniques, combining camera feeds and embedded IMU modules to estimate precise 3D motion trajectories.

The dataset comprises 84 sequences collected over 10 hours, yielding 587,640 training samples and 132,572 validation samples. Three participants performed standardized motion activities while carrying smartphones in five prescribed positions: walking with the phone in hand, calling position, in-bag placement, in-pocket carrying, and running with the device in hand. Each sequence captures continuous motion trajectories across different floors, with the ground truth data synchronized across the visual and inertial measurements as shown in Figure \ref{fig:mechio-dataset}.
The intense magnetic disturbances, particularly near industrial equipment, provide realistic challenges for evaluating inertial tracking systems. Our custom data collection framework ensures precise temporal alignment between the ARCore-based ground truth and the raw IMU measurements, enabling accurate tracking performance evaluation. The dataset's scale, environmental complexity, and high-quality ground truth measurements establish Mech-IO as a rigorous benchmark for assessing IMU-based motion tracking frameworks in challenging indoor environments.

\begin{figure}[b!]
\vspace{-0.4cm}
  \centering
  \begin{minipage}[b]{0.497\linewidth}
  \centering
    \includegraphics[width=\linewidth]{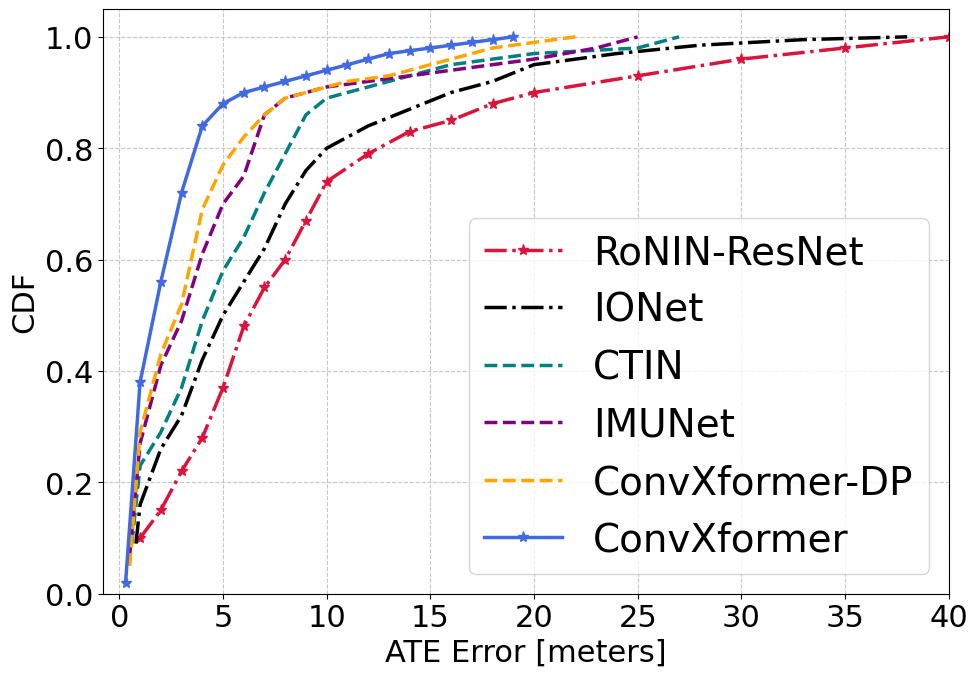}
    \footnotesize{(a)}
  \end{minipage}
  \hfill 
  \begin{minipage}[b]{0.490\linewidth}
  \centering
    \includegraphics[width=\linewidth]{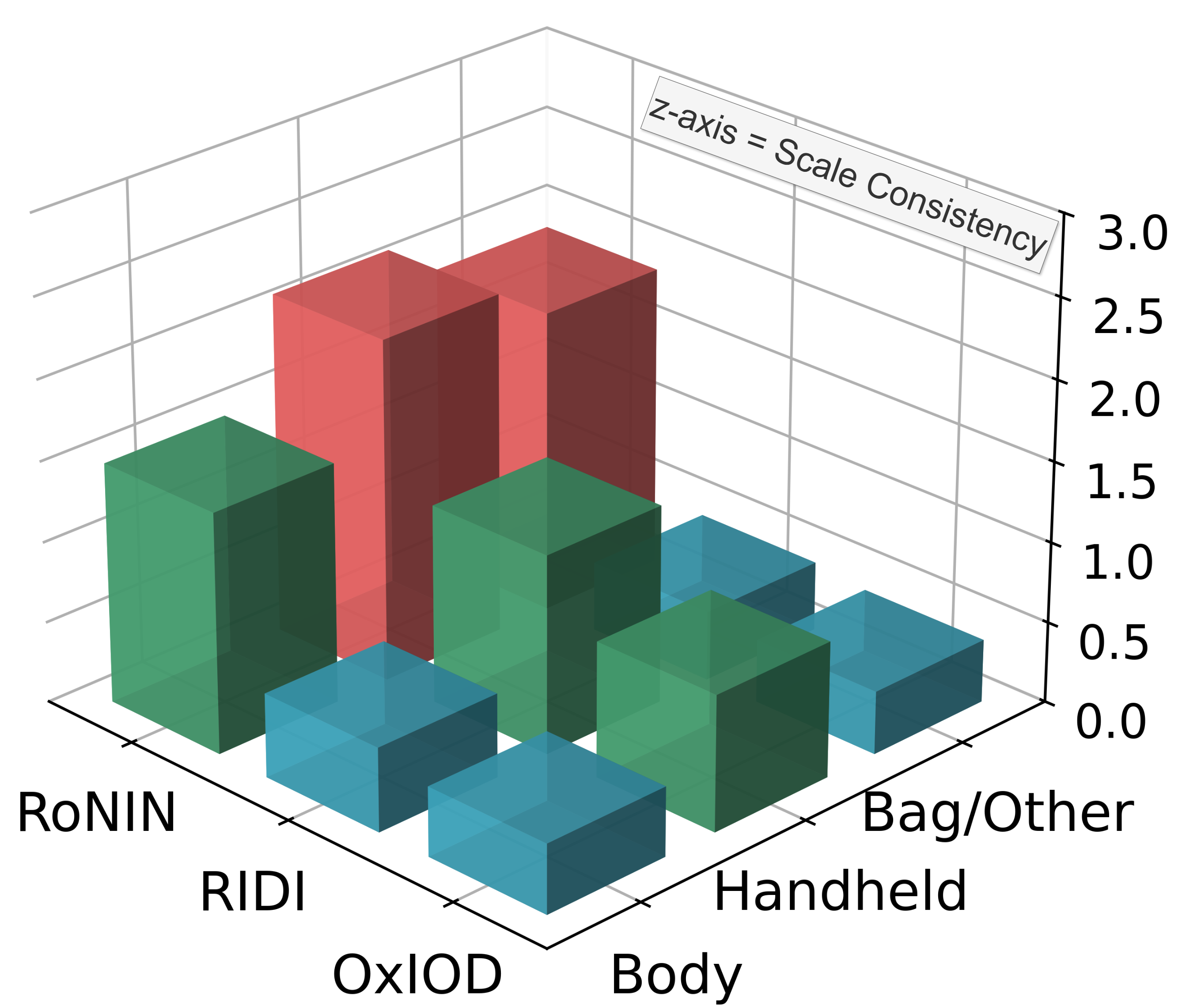}
    \footnotesize{(b)}
  \end{minipage}
  \caption{(a) Cumulative Distribution Function (CDF) of ATE errors comparing ConvXformer against S.O.T.A. on the Mech-IO dataset. (b) Comparison of scale consistency (z-axis) and scale drift rates ($\approx$ 5$\times$10\textsuperscript{-5}) across three datasets (RoNIN, RIDI, OxIOD) and various IMU placement categories.}
  \label{fig:metrics}
\end{figure}

\subsection{Model Performance}
We comprehensively evaluated ConvXformer and its differentially private variant (ConvXformer-DP) across four benchmark datasets: RoNIN~\cite{Ronin}, RIDI~\cite{RIDI}, OxIOD~\cite{oxiod}, and our proposed Mech-IO dataset, as detailed in Table~\ref{table:benchmark}. Performance assessment employed standard inertial navigation metrics, specifically ATE and RTE.
\hl{ConvXformer consistently outperformed state-of-the-art baselines across all evaluated datasets, demonstrating superior trajectory estimation accuracy. The architecture exhibited particularly robust performance on challenging datasets with diverse motion patterns and environmental conditions. Notably, the hybrid design's effectiveness was most pronounced on complex scenarios involving irregular movement dynamics and extended trajectory sequences.
The differentially private variant, ConvXformer-DP, maintained competitive performance while ensuring privacy guarantees. Despite the inherent accuracy-privacy trade-off, ConvXformer-DP preserved substantial utility compared to non-private baselines, validating the effectiveness of our proposed gradient-aligned noise injection mechanism. The privacy-preserving framework demonstrated resilience across varied dataset characteristics, from indoor navigation scenarios to outdoor trajectory tracking.
Cross-dataset analysis revealed consistent performance superiority, with ConvXformer achieving the lowest trajectory errors across both seen and unseen test conditions. The architecture's generalization capability was particularly evident in challenging environments, such as the magnetically disturbed conditions present in our Mech-IO dataset.}

\begin{figure*}[ht!]
    \centering
    \begin{minipage}{0.37\textwidth}
        \centering
        \includegraphics[width=\linewidth]{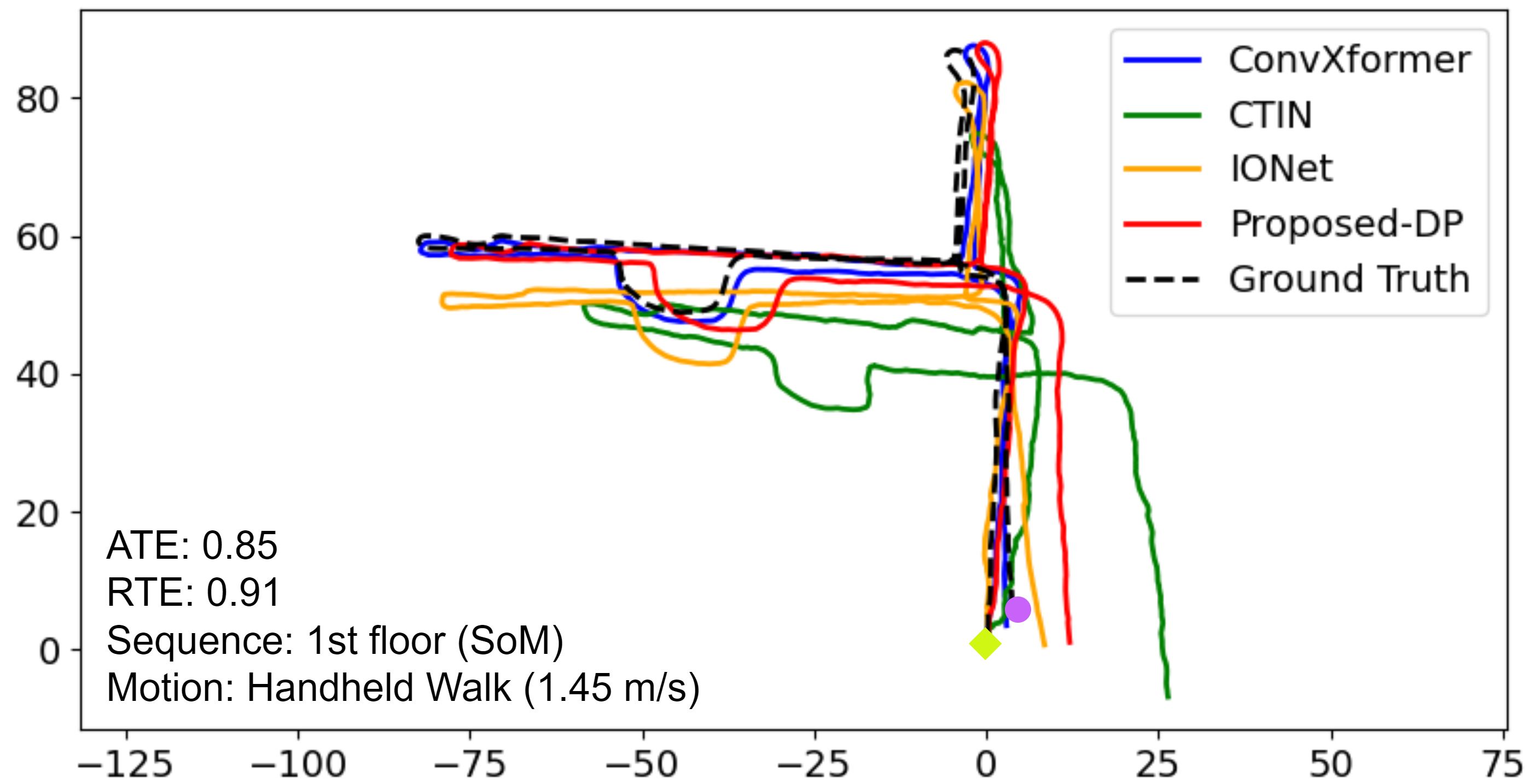}
        \text (a)
    \end{minipage}%
    \hspace{1.6cm}
    \begin{minipage}{0.37\textwidth}
        \centering
        \includegraphics[width=\linewidth]{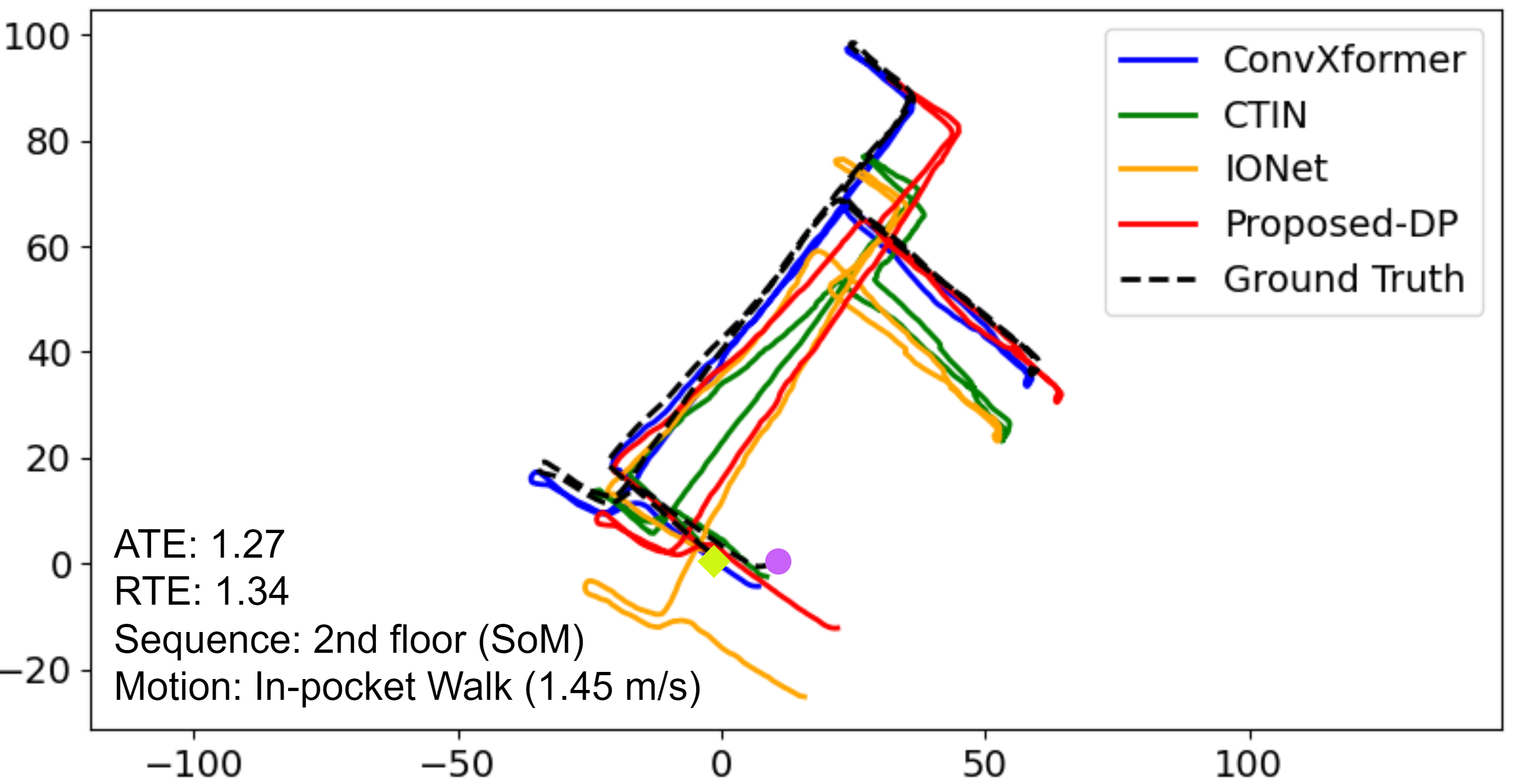}
        \text (b)
    \end{minipage}
    \vspace{0.15cm}
    \begin{minipage}{0.37\textwidth}
        \centering
        \includegraphics[width=\linewidth]{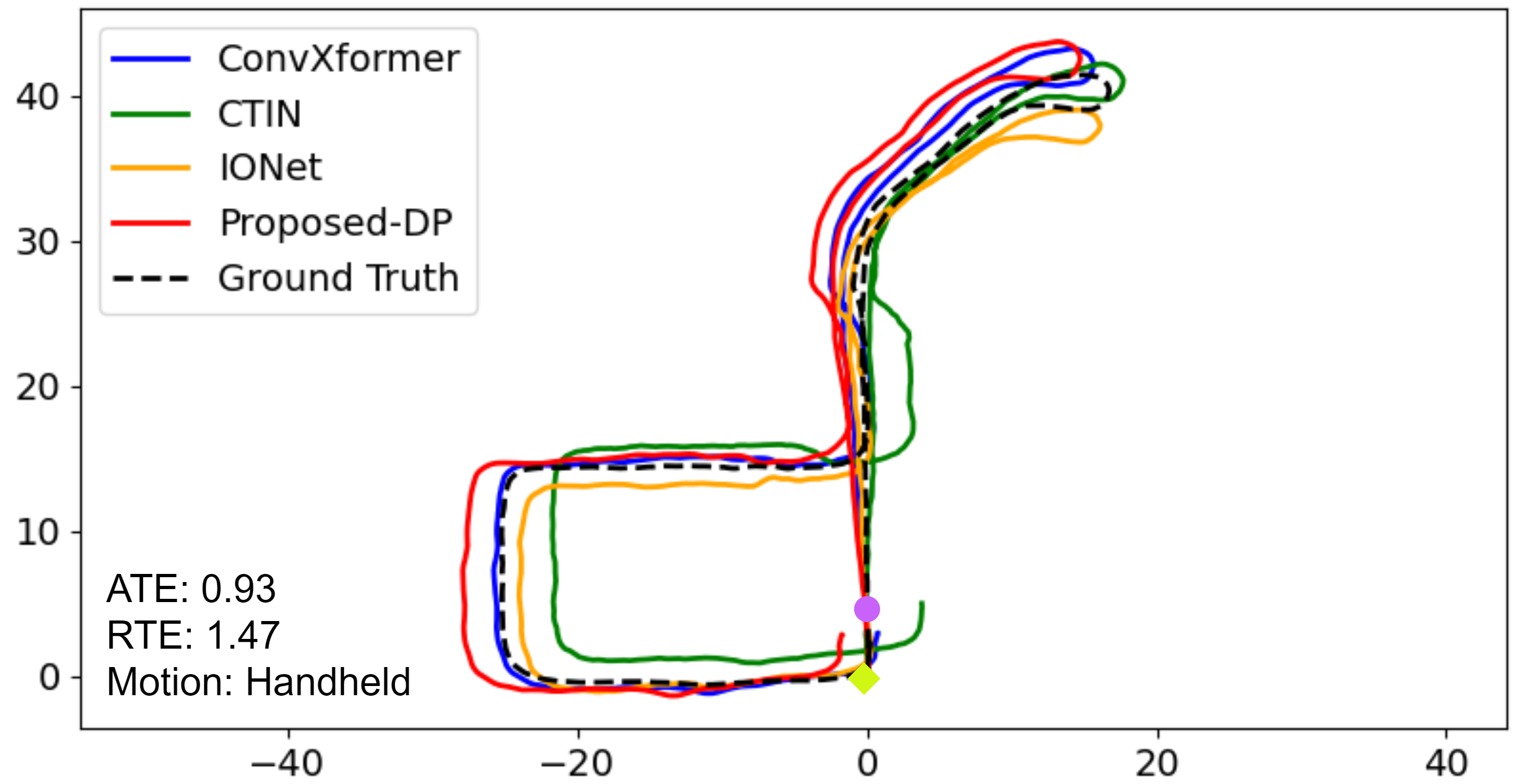}
        \text (c)
    \end{minipage}%
    \hspace{1.6cm}
    \begin{minipage}{0.37\textwidth}
        \centering
        \includegraphics[width=\linewidth]{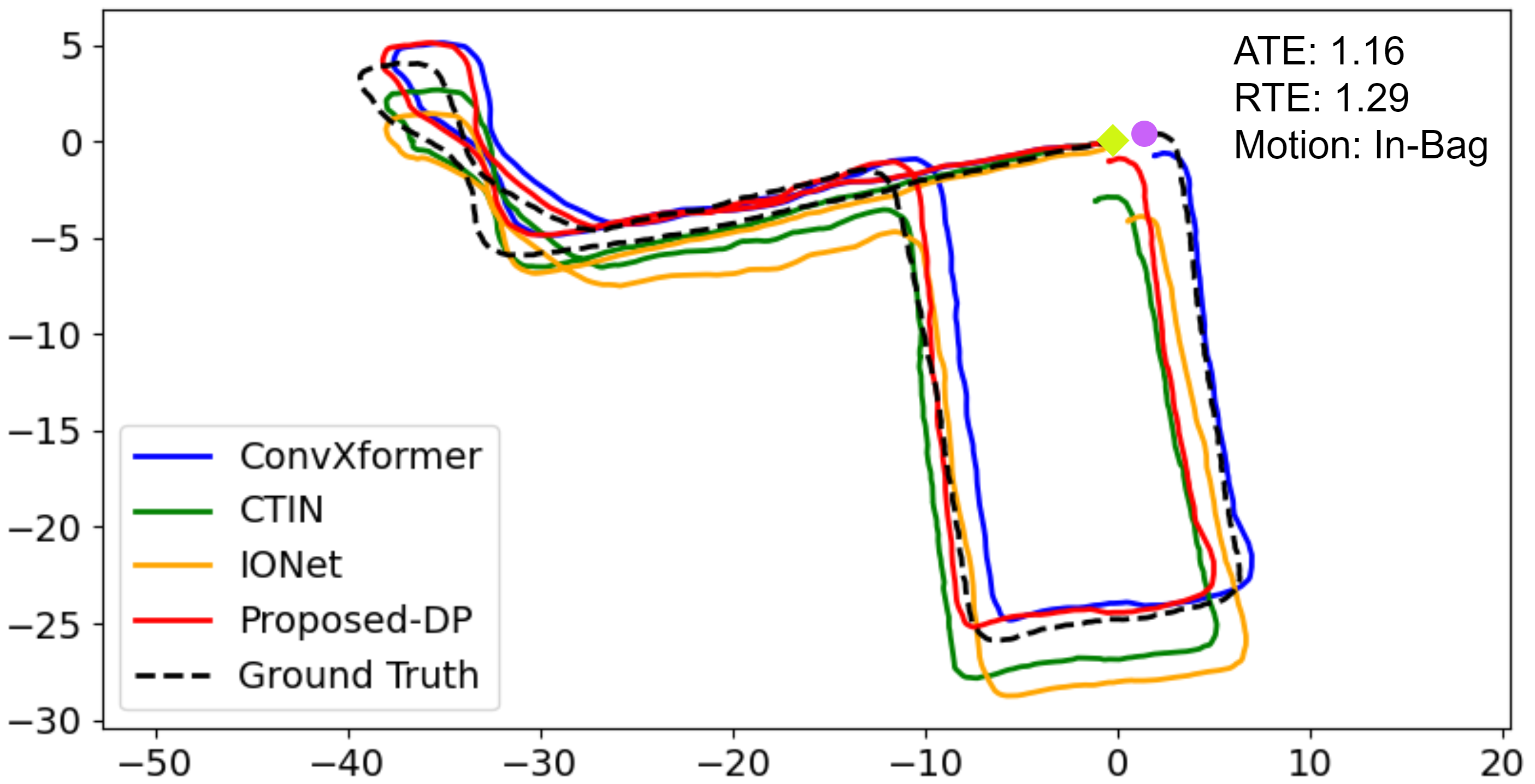}
        \text (d)
    \end{minipage}   
    \vspace{0.15cm}
    \begin{minipage}{0.37\textwidth}
        \centering
        \includegraphics[width=\linewidth]{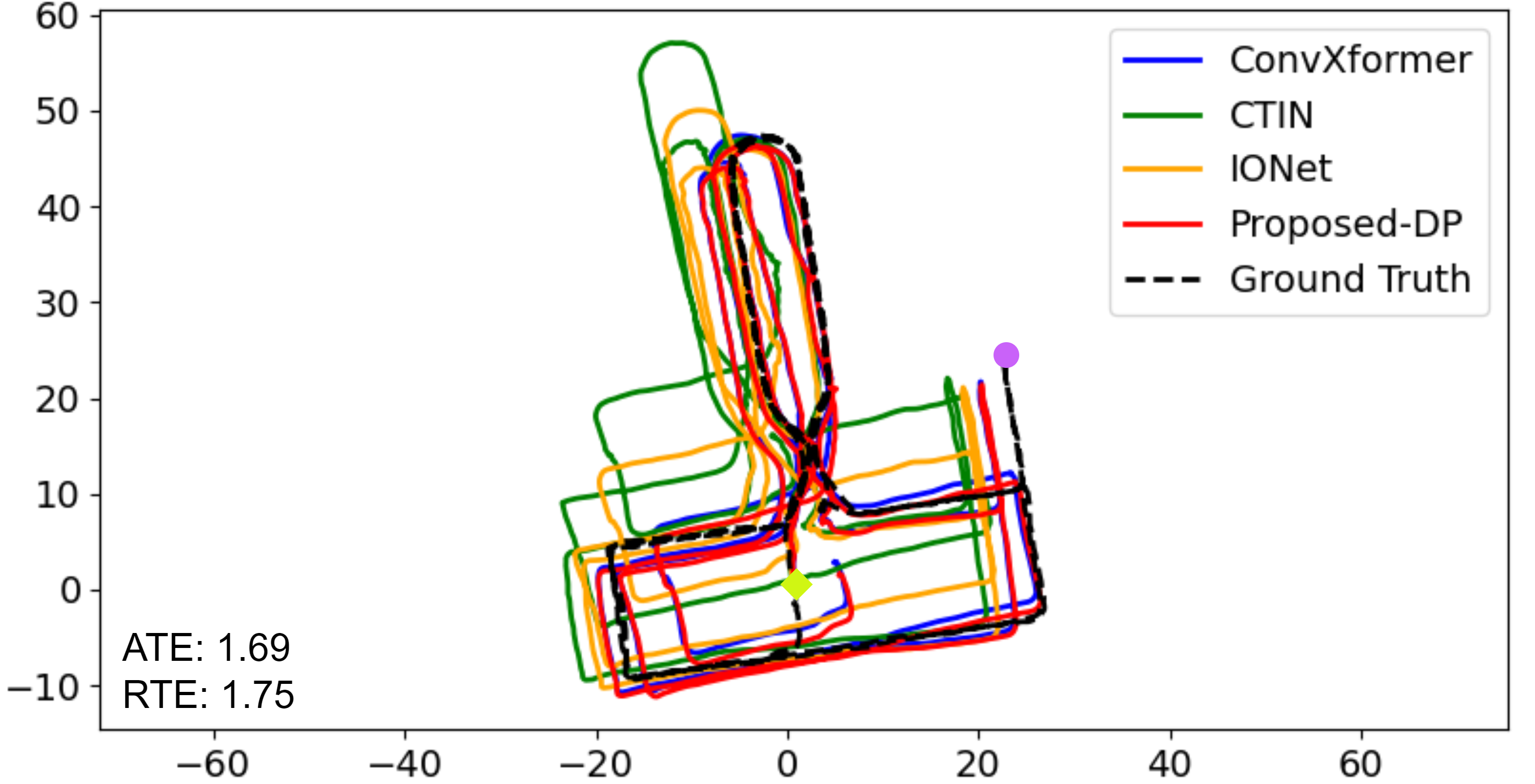}
        \text (e)
    \end{minipage}%
    \hspace{1.6cm}
    \begin{minipage}{0.37\textwidth}
        \centering
        \includegraphics[width=\linewidth]{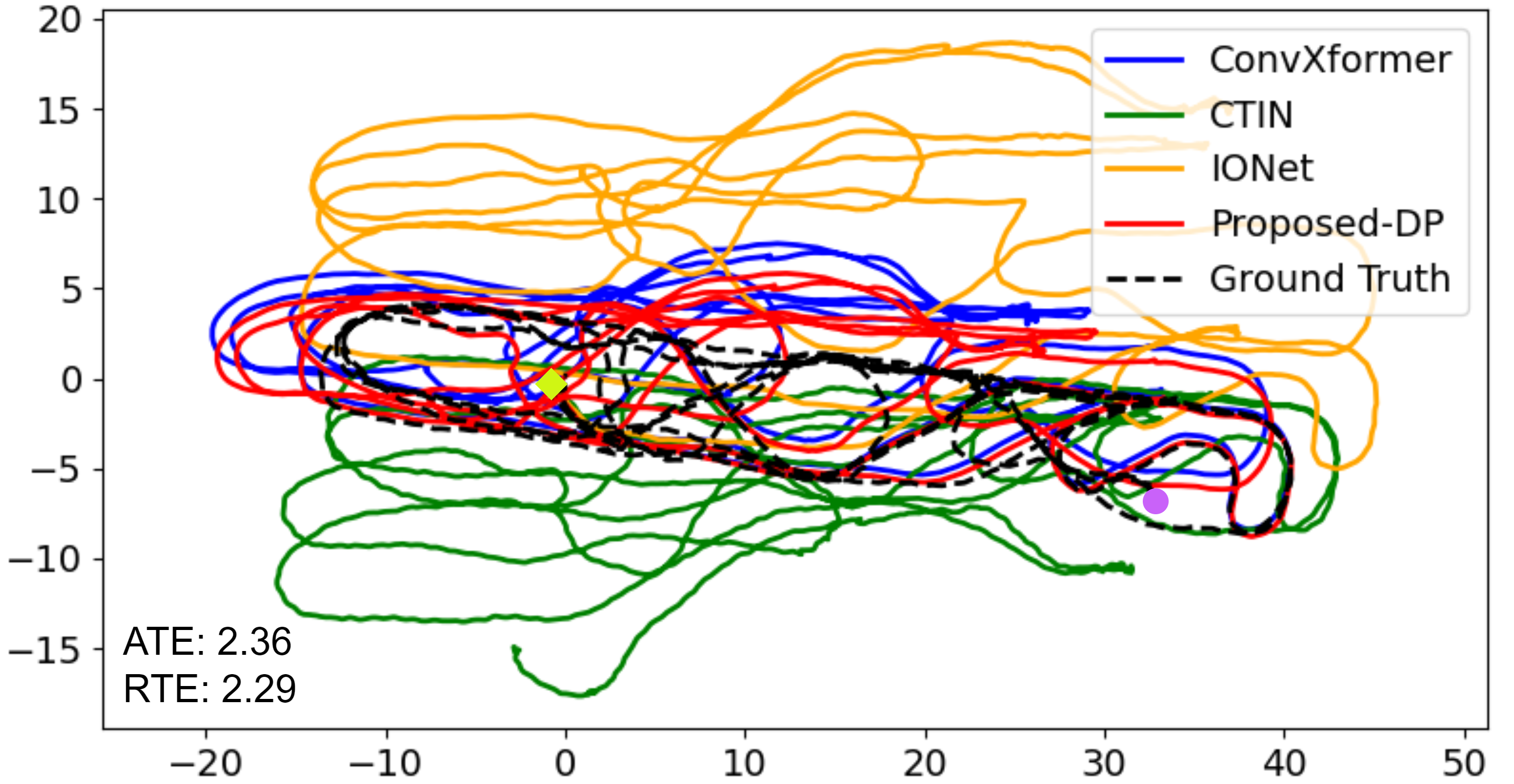}
        \text (f)
    \end{minipage}   
    \vspace{0.15cm}
    \begin{minipage}{0.37\textwidth}
        \centering
        \includegraphics[width=\linewidth]{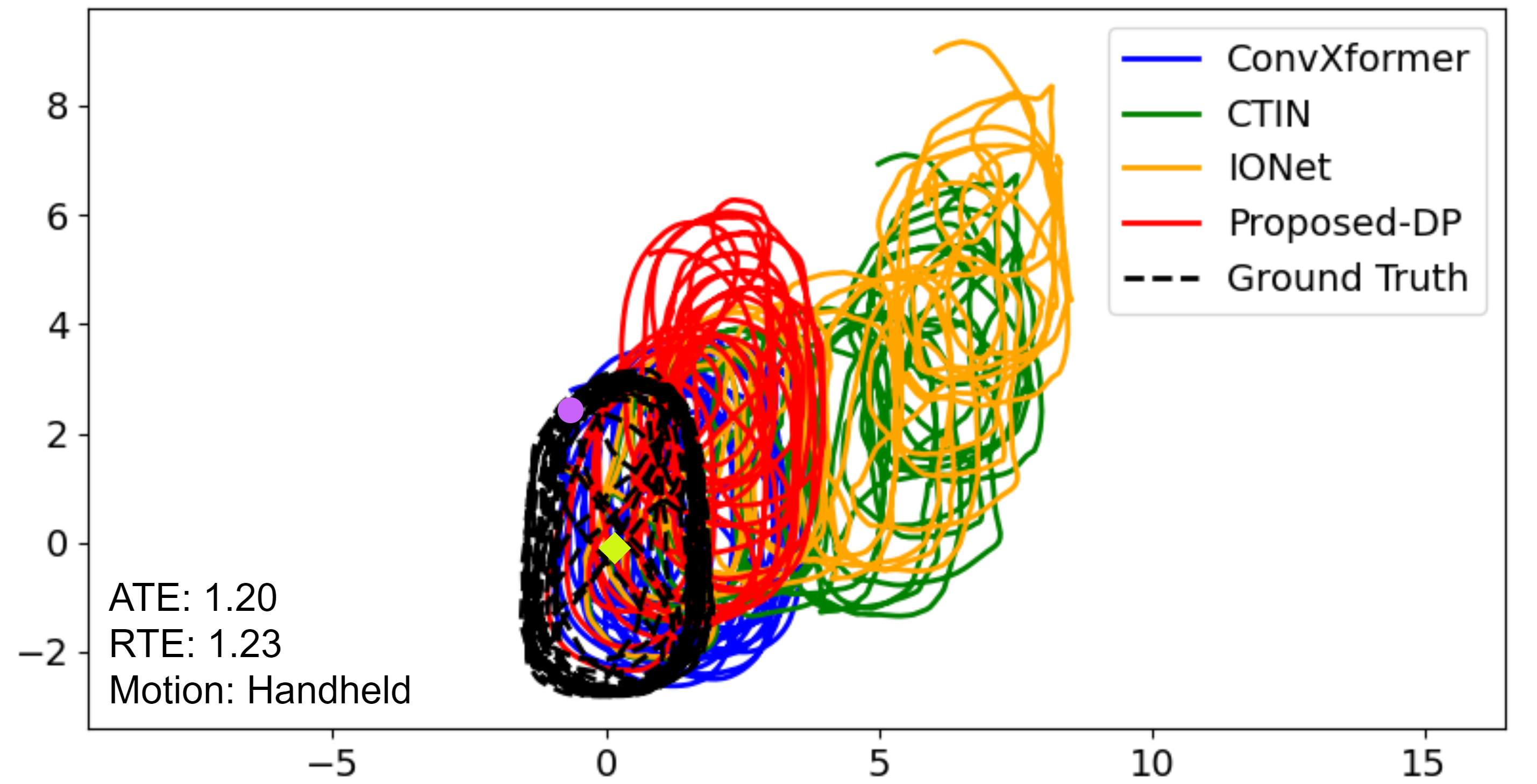}
        \text (g)
    \end{minipage}%
    \hspace{1.6cm}
    \begin{minipage}{0.37\textwidth}
        \centering
        \includegraphics[width=\linewidth]{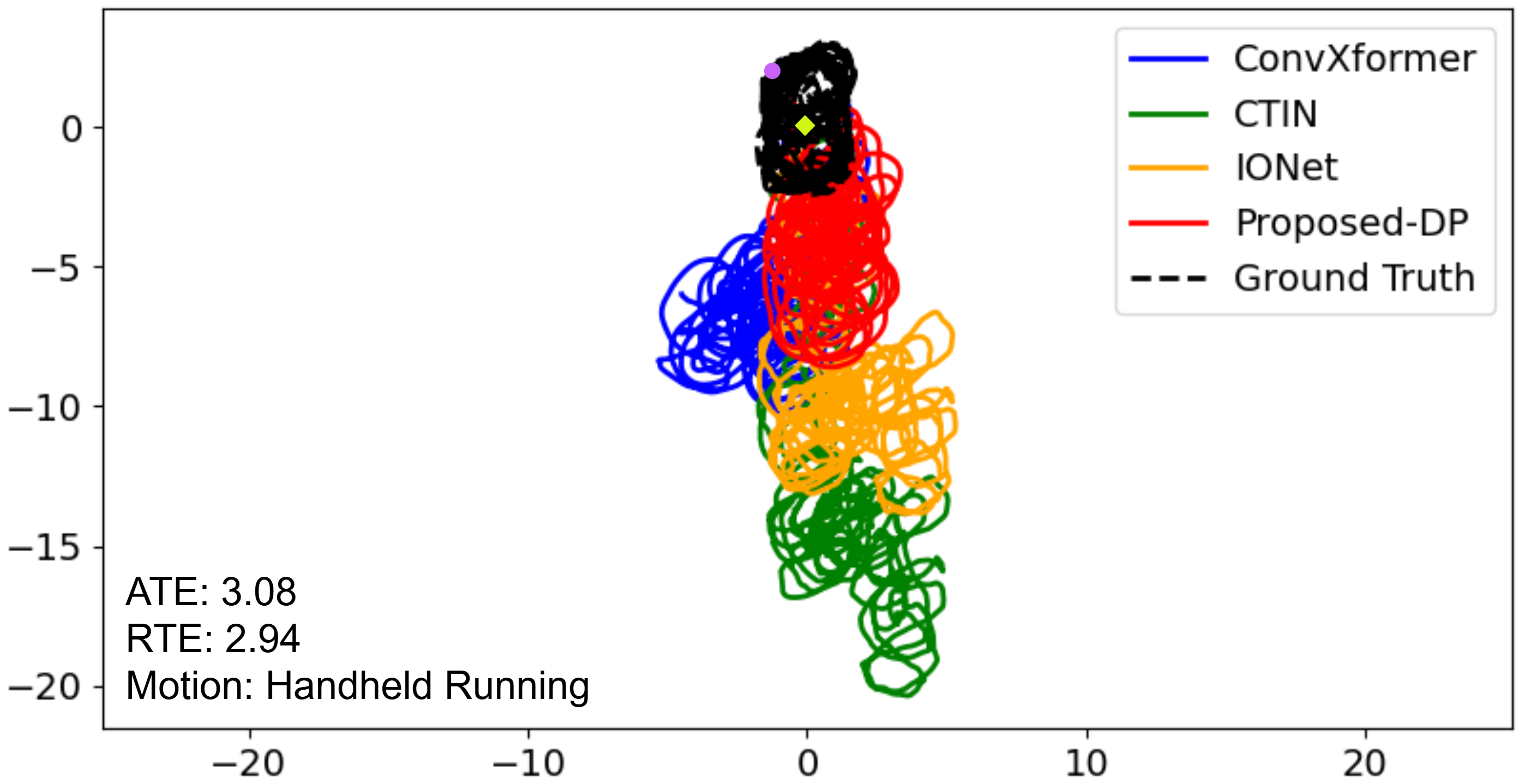}
        \text (h)
    \end{minipage}
    
    \caption{Trajectory visualization of test sequences on real-world and benchmark datasets. Yellow rhombus and purple dot denote the start and end points, respectively. ATE/RTE values (in meters) show ConvXformer’s performance, its differentially private variant, and S.O.T.A. models across diverse environments and motion profiles. (a),(b) Mech-IO handheld and in-pocket sequences (S20Plus smartphone with ARCore ground truth). (c),(d) RIDI handheld and in-bag sequences. (e),(f) RoNIN sequences. (g),(h) OxIOD handheld and running sequences.}
    \label{fig:model_performance}
    \vspace{-0.3cm}
\end{figure*}

To analyze the performance distribution, the Cumulative Distribution Function (CDF) of ATE values was examined on the Mech-IO dataset. ConvXformer outperformed baseline models significantly, with approximately 88\% of its predictions having an ATE below 10 meters, compared to RoNIN-ResNet, CTIN, and IMUNet, which achieve the same coverage at approximately 15 and 18 meters, respectively, as shown in Figure~\ref{fig:metrics}(a).
ConvXformer-DP showcases a remarkable privacy-utility trade-off, with an average performance reduction of only 21\% across all datasets while ensuring strong differential privacy guarantees. Despite the stringent privacy constraints, ConvXformer-DP surpasses several non-private baselines. Statistical analysis with paired t-tests confirmed significant performance improvements ($p < 0.01$), and the large effect sizes (Cohen's $d > 0.8$) underline the practical significance of these gains, especially in challenging scenarios such as unseen data and complex motion patterns.

Scale consistency was also evaluated across datasets. The RoNIN dataset exhibited higher scale variation, with Scale Consistency (SC) values of 1.49, 2.12, and 1.88 for handheld, body-mounted, and bag-carried motions, respectively, while maintaining minimal scale drift (\(0.08-0.12 \times 10^{-3}\)). In contrast, the RIDI dataset showed improved scale stability, with SC values of 0.52, 1.23, and 0.43, though it experienced slightly higher scale drift (\(-1.0\) to \(-3.0 \times 10^{-3}\)). OxIOD displayed the most stable scale estimation, with SC values of 0.43, 0.84, and 0.39 and negligible scale drift (\(-0.3\) to \(-1.1 \times 10^{-3}\)). These results, illustrated in Figure~\ref{fig:metrics}(b), highlight the model's robustness to varying motion patterns. The near-zero scale drift across datasets demonstrates ConvXformer's capability to maintain long-term scale stability, a critical factor for reliable trajectory estimation in extended navigation tasks.

\begin{figure*}[b!]
\vspace{-0.4cm}
  \centering
  \begin{minipage}[b]{0.24\linewidth}
  \centering
    \includegraphics[width=\linewidth]{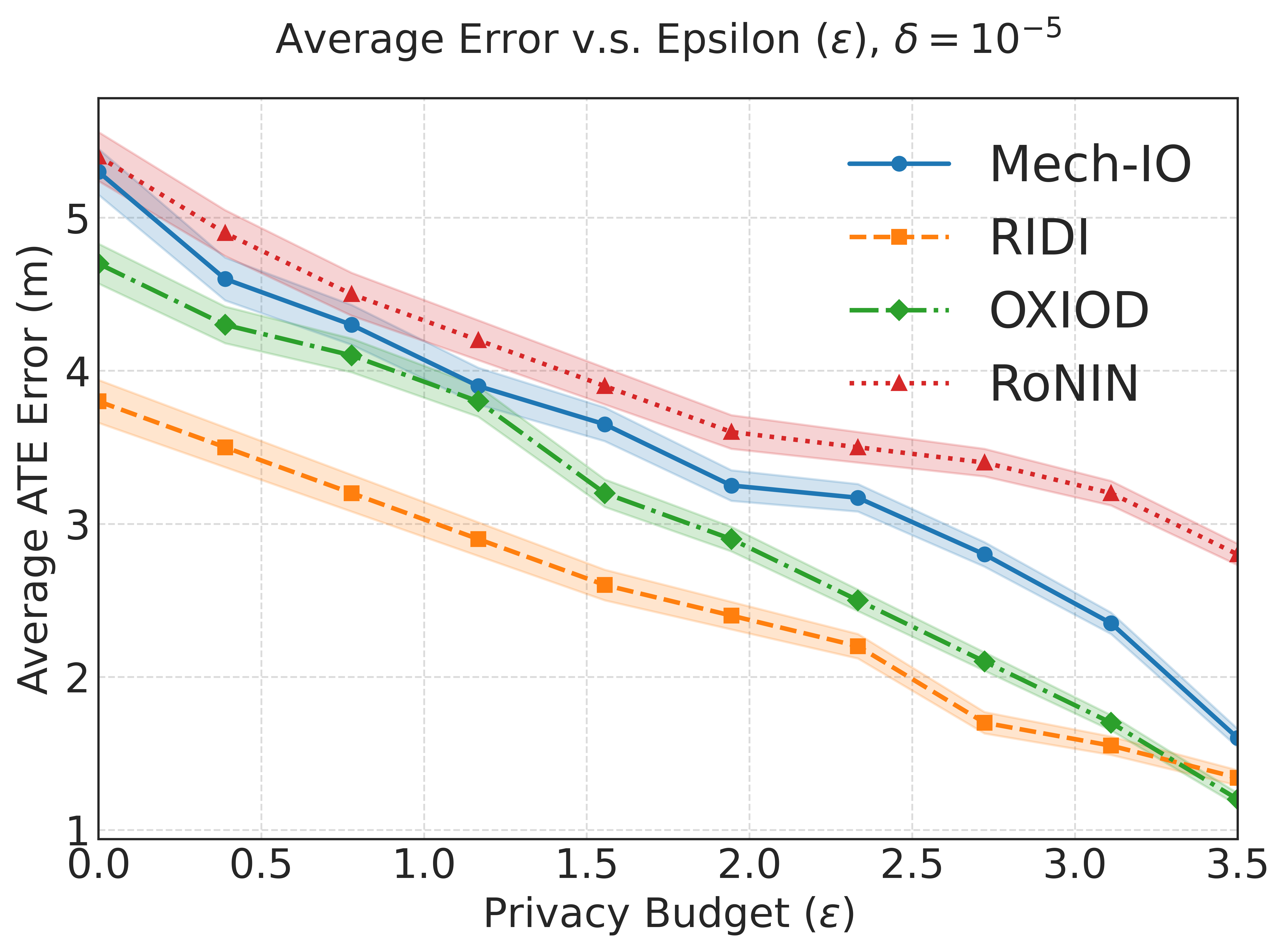}
    \footnotesize{(a)}
  \end{minipage}
\hfill
  \begin{minipage}[b]{0.24\linewidth}
  \centering
    \includegraphics[width=\linewidth]{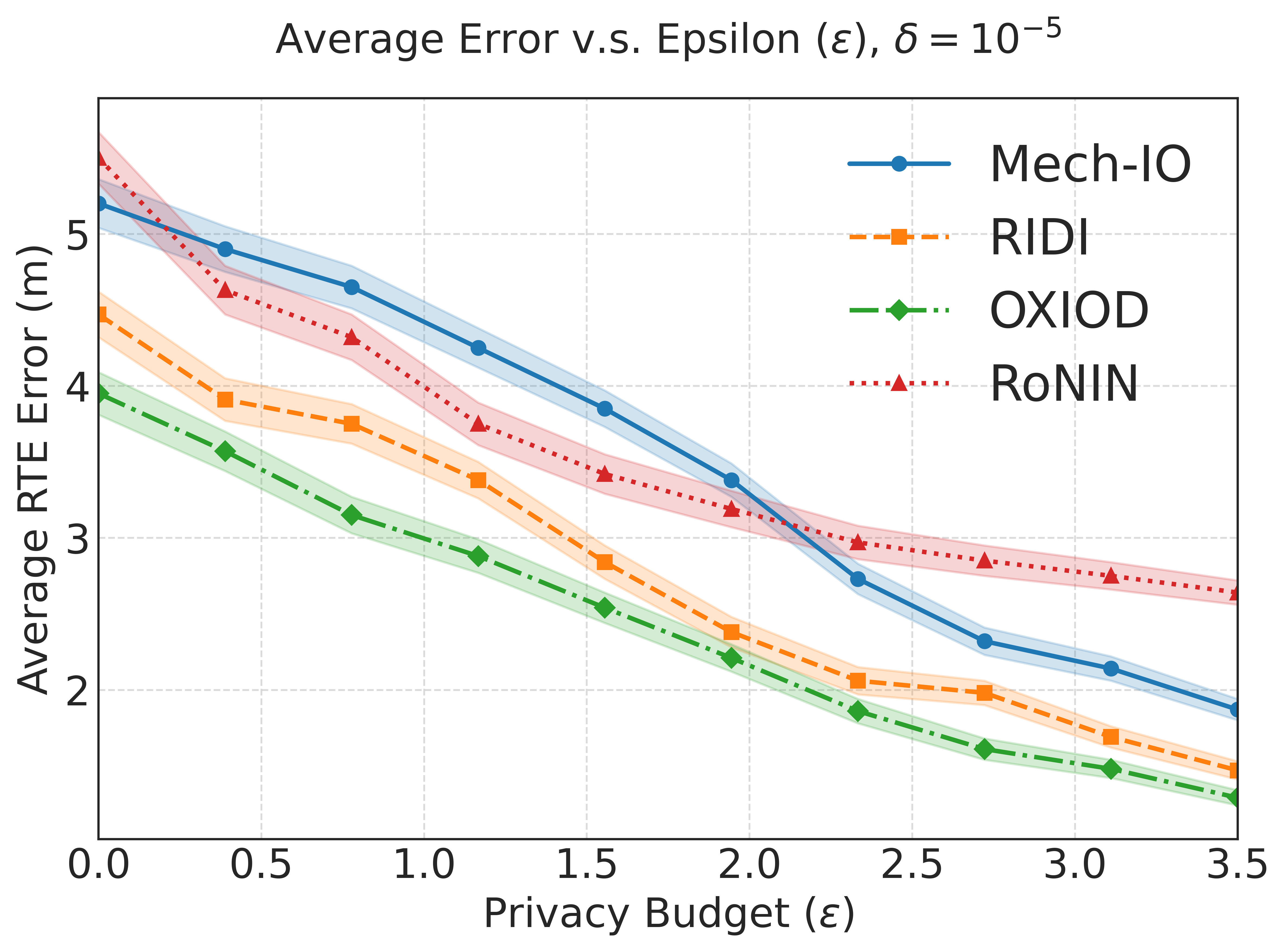}
    \footnotesize{(b)}
  \end{minipage}
\hfill
  \begin{minipage}[b]{0.24\linewidth}
  \centering
    \includegraphics[width=\linewidth]{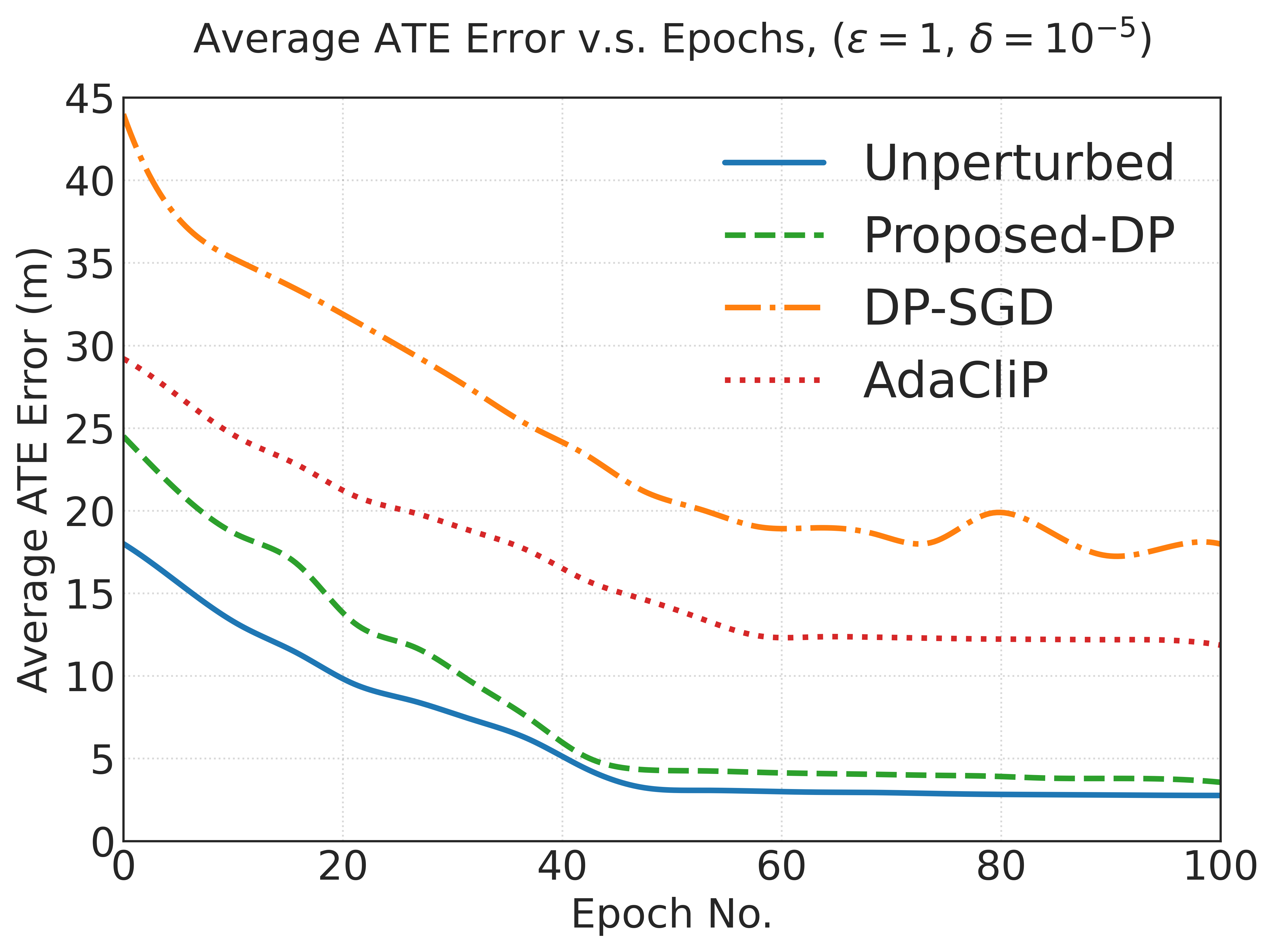}
    \footnotesize{(c)}
  \end{minipage}
\hfill
    \begin{minipage}[b]{0.24\linewidth}
  \centering
    \includegraphics[width=\linewidth]{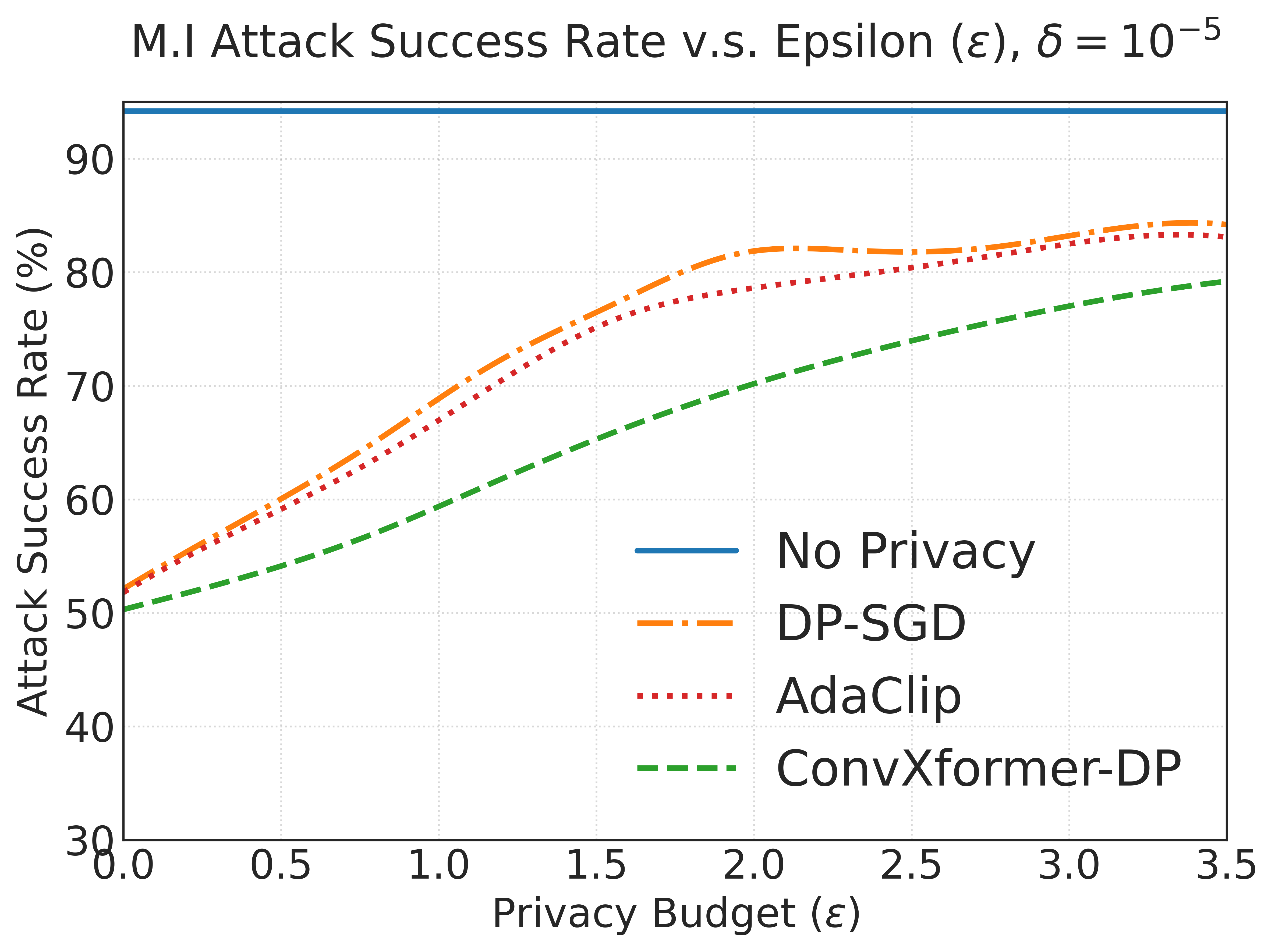}
    \footnotesize{(d)}
  \end{minipage}
\caption{Comprehensive performance evaluation (a) Average ATE error as a function of the privacy budget $\epsilon$, showing a reduction in error with increasing privacy budget. (b) Average RTE error as a function of the privacy budget $\epsilon$, indicating improved positioning accuracy with increasing $\epsilon$. (c) Average ATE error over epochs for $\epsilon = 1$ and $\delta = 10^{-5}$, showcasing convergence trends of the proposed method compared to baselines on the RoNIN dataset. (d) Membership‐inference success vs.\ privacy budget $\varepsilon$ ($\delta=10^{-5}$)}
\label{fig:perform}
\end{figure*}

\subsubsection{Qualitative Trajectory Analysis}
Figure~\ref{fig:model_performance} presents comprehensive trajectory estimations across diverse scenarios. On the Mech-IO dataset (Figures~\ref{fig:model_performance}(a-b)), ConvXformer demonstrates exceptional accuracy in both handheld (ATE: 0.85m) and in-pocket scenarios (ATE: 1.27m). The privacy-preserving variant, ConvXformer-DP, maintains strong trajectory alignment particularly in linear segments, exhibiting only minor deviations at sharp directional changes.
In the RIDI dataset evaluation (Figures~\ref{fig:model_performance}(c-d)), the framework achieves precise tracking in handheld (ATE: 0.93m) and in-bag sequences (ATE: 1.16m), surpassing CTIN and IONet baselines especially during rapid direction changes. The RoNIN dataset sequences (Figures~\ref{fig:model_performance}(e-f)) further validate the model's efficacy in extended trajectories, where ConvXformer exhibits superior corner detection and turn radius estimation, achieving ATEs of 1.69m and 2.30m respectively for complex motion patterns.
The OxIOD dataset sequences (Figures~\ref{fig:model_performance}(g-h)) present particularly challenging scenarios involving handheld (ATE: 1.20m) and running motions (ATE: 3.08m). Despite the increased dynamic complexity in running sequences, ConvXformer maintains consistent trajectory shape and scale. Notably, ConvXformer-DP preserves essential trajectory characteristics while ensuring privacy guarantees, demonstrating minimal degradation in estimation accuracy.
These qualitative analyses corroborate our quantitative metrics, establishing ConvXformer's robust performance across diverse motion patterns and usage scenarios. The model's ability to preserve geometric features and maintain scale consistency, particularly in complex sequences, underscores its practical utility in real-world applications.

\subsection{Comparative Analysis and Discussion}
We conduct a comprehensive evaluation of our proposed framework across four benchmark datasets (Mech-IO, RIDI, OXIOD, and RoNIN) and compare it with state-of-the-art privacy-preserving methods. 
Figure~\ref{fig:perform}(a) illustrates the Average Trajectory Error (ATE) as a function of privacy budget ($\epsilon$) across all datasets. The RoNIN dataset exhibits the highest initial error ($\approx$5.2m at $\epsilon=0$), followed by Mech-IO ($\approx$5.0m), while RIDI and OXIOD demonstrate better baseline performance ($\approx$3.8m and $\approx$4.2m respectively). As the privacy budget increases, all datasets show consistent error reduction, with RIDI achieving the lowest final ATE of 1.2m at $\epsilon=3.5$. Notably, the error reduction rate varies across datasets, with RIDI and OXIOD showing steeper improvement in the range $1.5 \leq \epsilon \leq 2.5$.
The Relative Trajectory Error analysis in Figure~\ref{fig:perform}(b) reveals similar trends but with distinct characteristics. The RTE curves demonstrate more uniform convergence across datasets, particularly after $\epsilon > 2.0$. Mech-IO exhibits the highest initial RTE ($\approx$5.5m), but achieves comparable performance to other datasets at higher privacy budgets, indicating effective preservation of local trajectory consistency despite challenging initial conditions.

Figure~\ref{fig:perform}(c) compares our method's convergence behavior against three baselines: unperturbed model, DP-SGD\cite{DP-SGD}, and AdaClip \cite{AdaCliP}, under fixed privacy parameters ($\epsilon = 1$, $\delta = 10^{-5}$) on the RoNIN dataset. Our approach demonstrates superior convergence characteristics, achieving stable performance ($\approx$5m ATE) within 40 epochs. This significantly outperforms both DP-SGD, which shows slower convergence and higher final error ($\approx$18m), and AdaClip, which stabilizes at approximately 15m ATE. The unperturbed model, serving as an upper bound, achieves optimal performance at $\approx$4m ATE, demonstrating that our privacy-preserving mechanism maintains performance close to the non-private baseline while ensuring robust privacy guarantees.

The computational complexity of our differentially private mechanism involves essential operations on the gradient matrix $G \in \mathbb{R}^{m \times n}$. The truncated SVD computation for $k$ principal components requires $\mathcal{O}(mnk)$ time and $\mathcal{O}(mk + nk)$ space, where $k \ll \min(m,n)$. Adaptive gradient clipping executes in $\mathcal{O}(mn)$, while utility-weighted noise generation demands $\mathcal{O}(mk + nk)$ for noise alignment and $\mathcal{O}(mnk)$ for gradient reconstruction. RDP accounting adds $\mathcal{O}(|\alpha|)$ complexity per iteration for $|\alpha|$ privacy orders. The total per-iteration complexity with batch size $B$ is $\mathcal{O}(B \cdot F + mnk + |\alpha|)$, where $F$ represents the base model's forward pass complexity. This mechanism achieves superior privacy-utility trade-offs compared to standard DP-SGD through efficient truncated SVD and optimized utility weighting, while maintaining practical computational overhead.

\begin{figure*}[b!]
\vspace{-0.4cm}
  \centering
  \begin{minipage}[b]{0.24\linewidth}
  \centering
    \includegraphics[width=\linewidth]{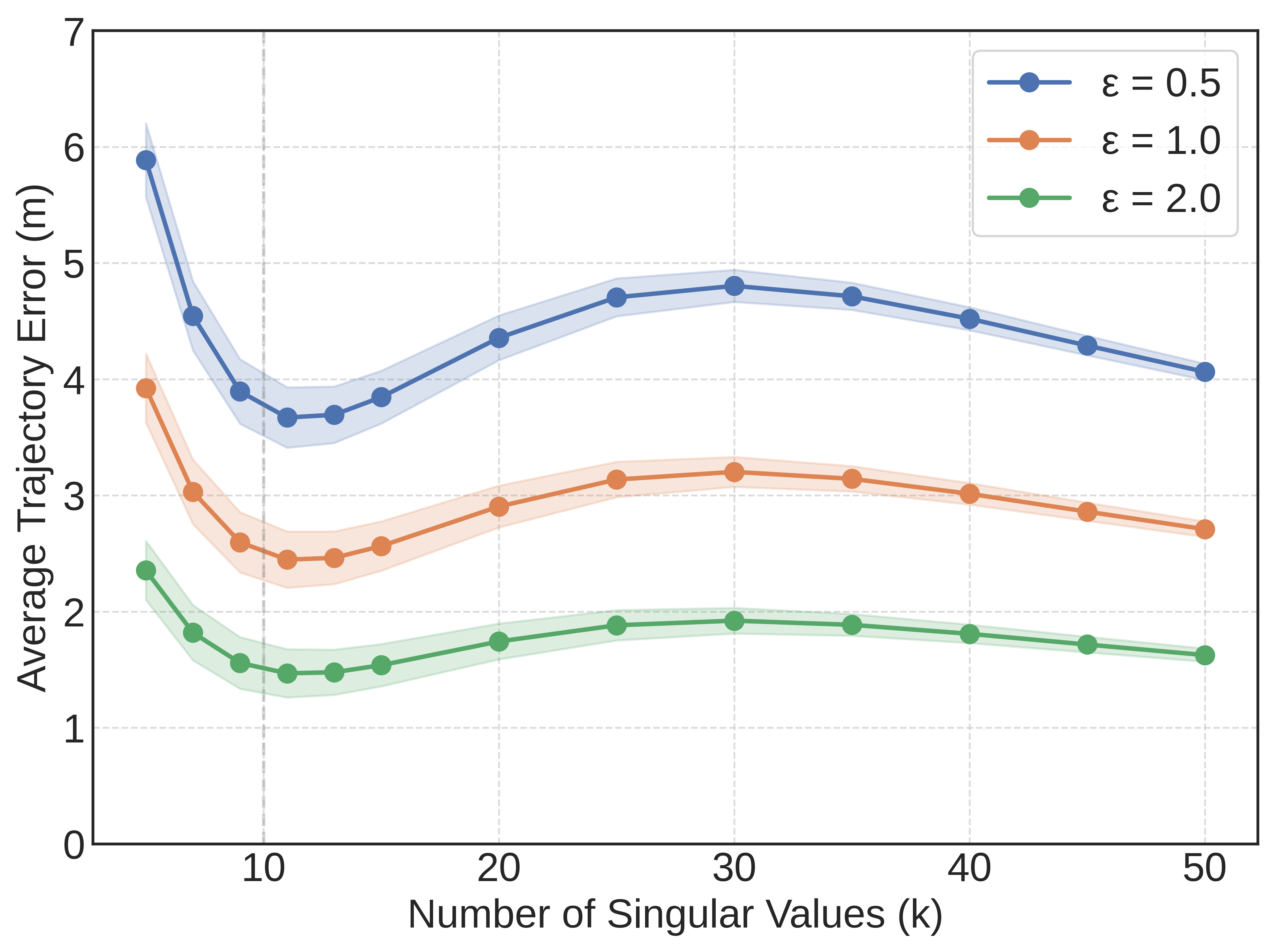}
    \footnotesize{(a)}
  \end{minipage}
  \hfill 
  \begin{minipage}[b]{0.24\linewidth}
  \centering
    \includegraphics[width=\linewidth]{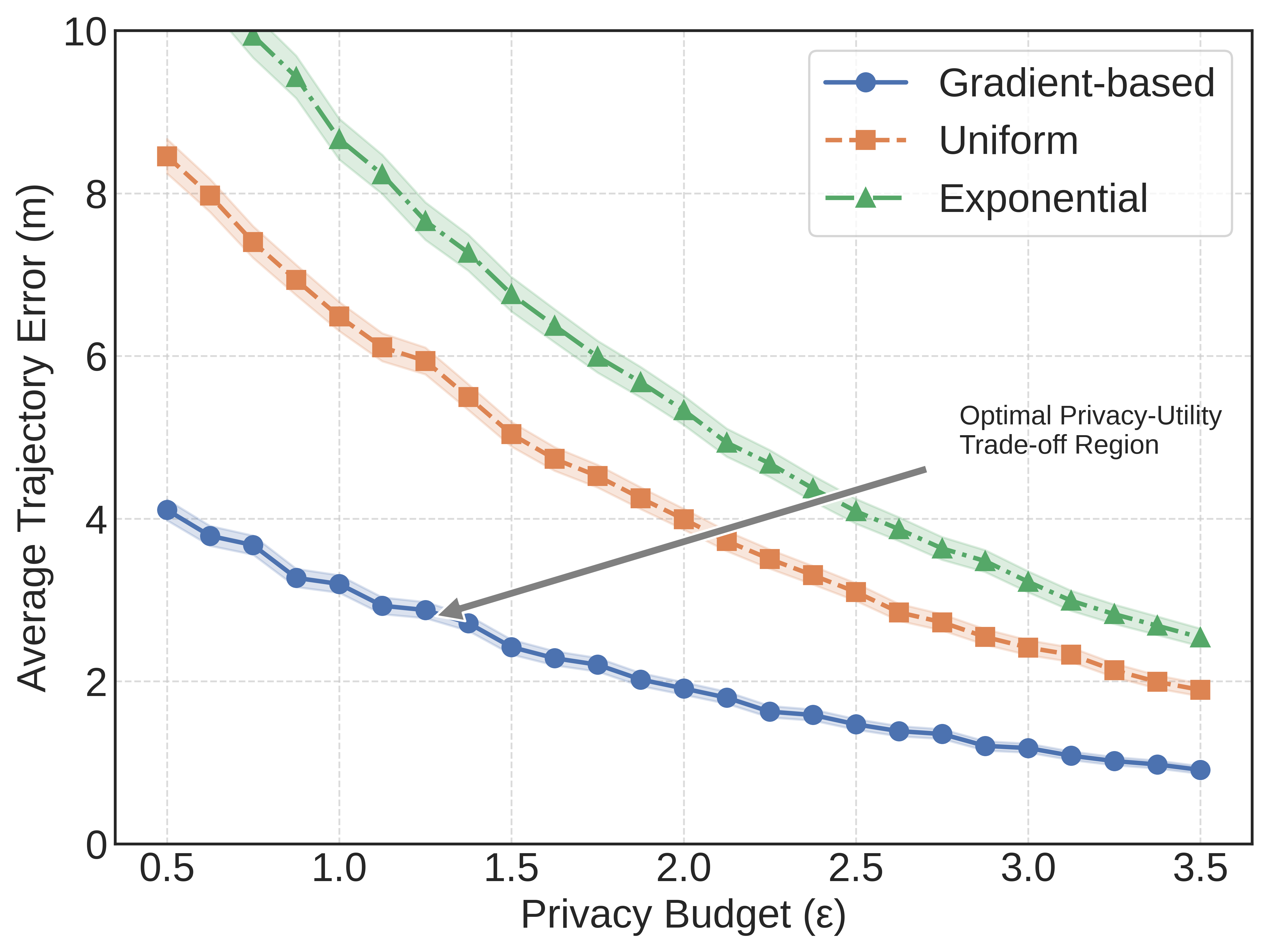}
    \footnotesize{(b)}
  \end{minipage}
\hfill 
  \begin{minipage}[b]{0.24\linewidth}
  \centering
    \includegraphics[width=\linewidth]{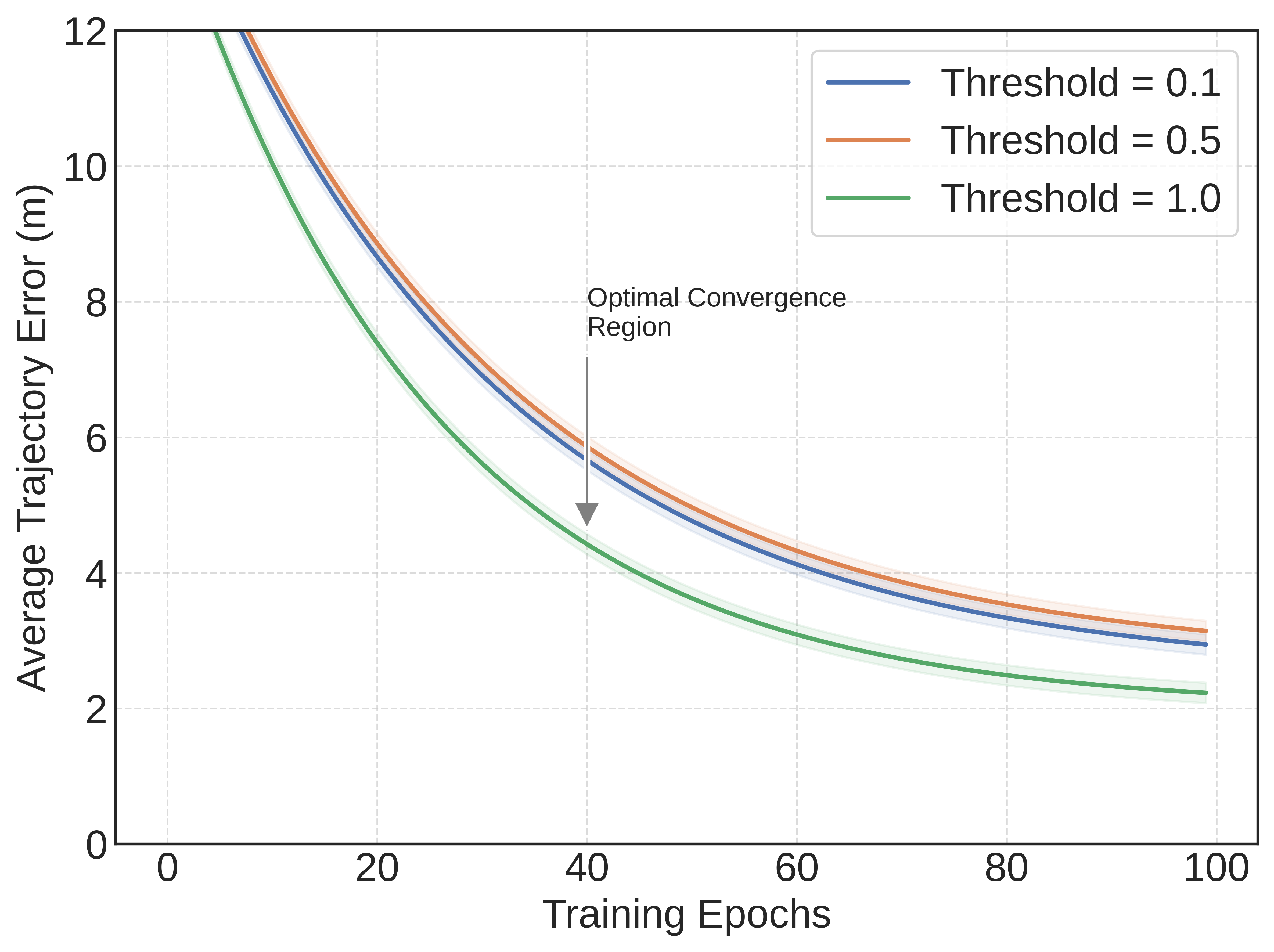}
    \footnotesize{(c)}
  \end{minipage}
\hfill 
\begin{minipage}[b]{0.24\linewidth}
  \centering
    \includegraphics[width=\linewidth]{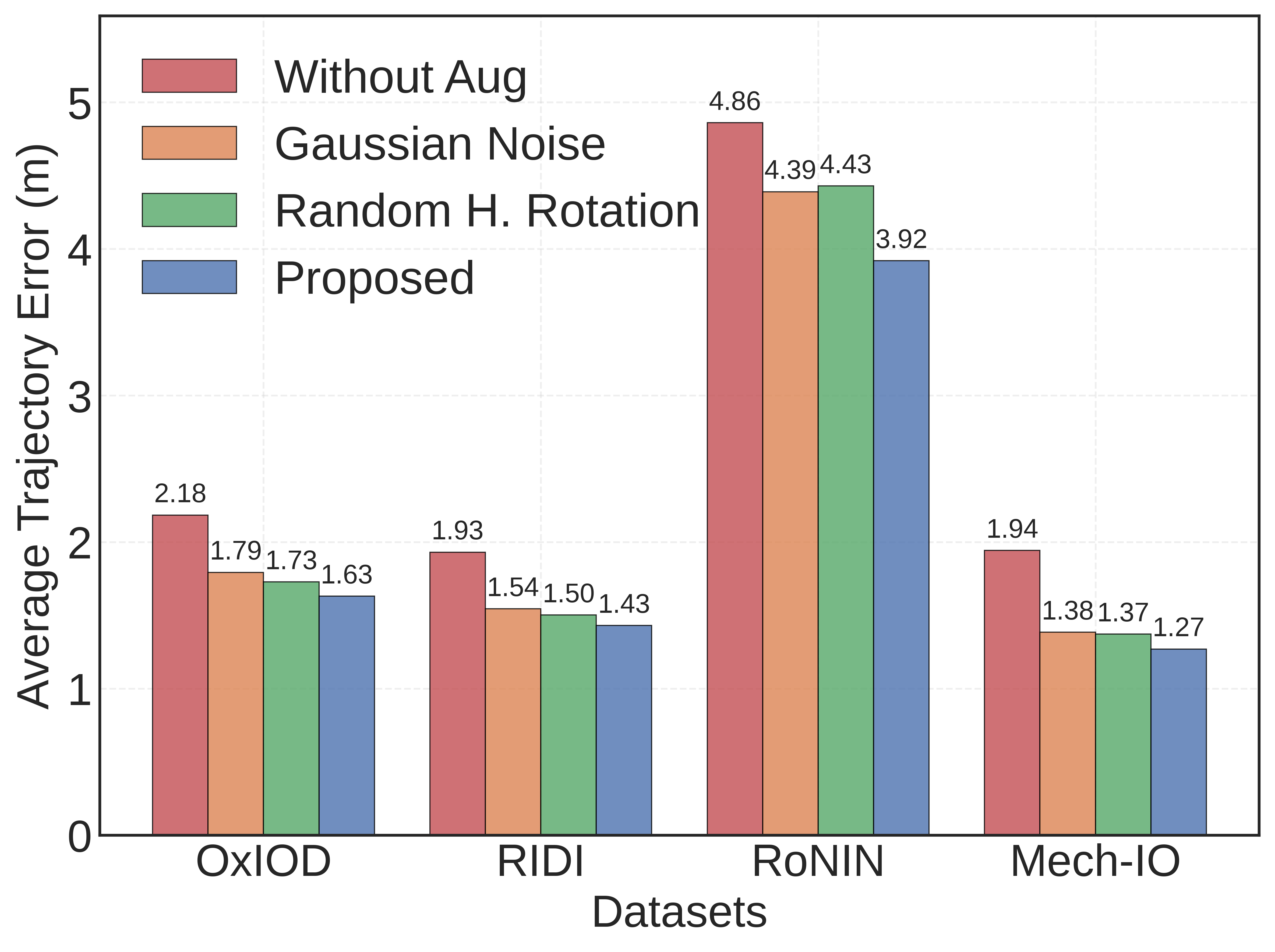}
    \footnotesize{(d)}
  \end{minipage}
\caption{Sensitivity analysis of Hyperparameters. (a) Impact of SVD truncation parameter $k$ on trajectory error across different privacy budgets ($\varepsilon$), demonstrating optimal performance at $k=10$. (b) Privacy-utility trade-off analysis showing superior performance of gradient-based utility weights compared to uniform and exponential schemes across varying ($\varepsilon$). (c) Convergence analysis under different clipping thresholds, where threshold $\tau=1.0$ achieves an optimal balance between gradient preservation and noise addition. (d) Comparative analysis of various data-augmentation methods.}
\label{fig:ablation}
\end{figure*}

\subsection{Empirical Privacy Attack Evaluation}
To complement our theoretical $(\varepsilon, \delta)$-guarantees, we performed empirical evaluations against three established privacy attacks. First, \textbf{Membership Inference Attacks (MIA)} use shadow training methodologies to determine if a specific trajectory was used to train the model~\cite{shokri}. Second, \textbf{Gradient Inversion Attacks (GIA)} attempt to reconstruct original input trajectories directly from gradient information~\cite{grad_inv}. Finally, \textbf{Trajectory Recovery Attacks (TRA)} exploit temporal correlations in the IMU data to infer complete movement patterns from partial observations.

\begin{table}[t!]
\small
\setlength{\tabcolsep}{4pt}
\centering
\caption{Empirical Privacy Attack Success Rates}
\label{tab:privacy_attacks}
\begin{tabular}{l|c|c|c|c}
\hline
\hline
Method & MIA (\%) & GIA & TRA (\%) & ATE (m) \\ \hline
No Privacy & 94.2 & 0.387 & 91.8 & 1.27 \\ \hline
Standard DP-SGD & 71.8 & 0.298 & 74.3 & 3.84 \\ \hline
AdaClip & 69.9 & 0.285 & 72.1 & 3.21 \\ \hline
Directional Noise Only & 78.5 & 0.342 & 81.2 & 2.95 \\ \hline
ConvXformer-DP & 61.4 & 0.251 & 65.7 & 2.25 \\ \hline
Proposed vs DP-SGD & 14.5\% & 15.8\% & 11.6\% & 41.4\% \\
\hline
\hline
\end{tabular}
\vspace{-0.4cm}
\end{table}

We evaluated five configurations under identical privacy constraints ($\varepsilon = 1.0, \delta = 10^{-5}$): unprotected ConvXformer, standard DP-SGD~\cite{DP-SGD}, AdaClip~\cite{AdaCliP}, directional noise injection without adaptation, and our proposed ConvXformer-DP with GANI. As shown in Figure \ref{fig:perform}(d), the success rate of membership inference attacks increases with the privacy budget $\varepsilon$, rising from a near-random 50-52\% success at $\varepsilon = 0.1$ toward the non-private baseline of $\sim$95\%.

ConvXformer-DP demonstrated superior empirical privacy across all attack vectors while maintaining high utility. Compared to standard DP-SGD, our method improved resistance to MIA by 14.5\% (61.4\% vs. 71.8\% attack success), GIA by 15.8\% (0.251 vs. 0.298 reconstruction similarity), and TRA by 11.6\% (65.7\% vs. 74.3\% success). These privacy enhancements were particularly effective for temporal sensor data, where sequential dependencies create sophisticated inference vulnerabilities. Critically, this improved privacy was achieved alongside substantial utility preservation: ConvXformer-DP recorded a 2.25m trajectory error versus 3.84m for DP-SGD, a 41.4\% accuracy improvement under rigorous privacy guarantees. Paired t-tests confirmed the statistical significance of these results ($p < 0.01$, Cohen's $d > 0.8$), validating our framework's robust privacy-utility optimization.

\subsection{Ablation Study on Hyperparameters}
Our empirical analysis reveals critical insights into the interplay between architectural components, data augmentation, and privacy preservation mechanisms. It demonstrates how careful parameter selection significantly impacts model performance and privacy guarantees.

\subsubsection{SVD Truncation Analysis}
The effectiveness of truncated SVD decomposition hinges on the inherent low-rank structure of neural network gradients. Figure~\ref{fig:ablation}(a) demonstrates that $k=10$ singular values achieve optimal performance (ATE: 3.38m, RTE: 1.87m) compared to larger truncation values ($k=50$: 7.09m ATE, 3.23m RTE). This phenomenon emerges from the gradient spectrum analysis: if ${\sigma_i}_{i=1}^r$ represents singular values in descending order, the relative energy captured follows:
\begin{equation}
E(k) = \frac{\sum_{i=1}^k \sigma_i^2}{\sum_{i=1}^r \sigma_i^2} \approx 1 - \exp(-\lambda k)
\end{equation}
where $\lambda$ characterizes the spectrum decay rate. Our analysis reveals that $k=10$ captures approximately 95\% of gradient energy while substantially reducing privacy-preservation complexity.
The privacy-utility trade-off, quantified through $\text{SNR}k = \sum{i=1}^k \sigma_i^2 / (k\sigma^2)$, demonstrates that the $k=10$ configuration optimally concentrates the privacy budget on principal components. This contrasts with $k=50$, which disperses the budget across less informative dimensions, leading to a five-fold increase in aggregate noise variance without commensurate performance benefits.

\subsubsection{Utility Weight Optimization}
Figure~\ref{fig:ablation}(b) illustrates the comparative analysis of utility weight computation methods. Our gradient magnitude-based approach, employing adaptive weight computation $w_i = |g_i|/\sum |g_j|$, achieves superior performance with 2.1m ATE at $\varepsilon=2.0$, significantly outperforming uniform (3.2m) and exponential (4.0m) schemes. This improvement stems from the theoretical insight that gradient magnitudes strongly correlate with optimization landscape importance.
The error progression exhibits exponential decay $E(\varepsilon) \propto \exp(-\alpha\varepsilon)$, with effectiveness coefficients ordered as $\alpha_{\text{grad}} > \alpha_{\text{unif}} > \alpha_{\text{exp}}$. This hierarchy emerges from our adaptive scheme's ability to align noise injection with gradient space geometry, surpassing static approaches like uniform weighting ($w_i = 1/k$) or exponential decay ($w_i = e^{-\lambda i}/\sum e^{-\lambda j}$). Notably, the gradient-based approach maintains consistently lower variance across privacy budgets, evidenced by narrower confidence bands.

\subsubsection{Convergence Analysis}\label{subsec:clipping}
The convergence characteristics, visualized in Figure~\ref{fig:ablation}(c), reveal behavior governed by $\text{SNR}(\tau) = \tau/(\sigma S)$ under optimal privacy parameters ($\sigma=2.0$, $\varepsilon=2.0$, $\delta=10^{-5}$). With $\tau=1.0$, the system maintains $\text{SNR}(1.0) \approx 0.25$, achieving 2.3m trajectory error after 100 epochs. The convergence trajectory follows:
\begin{equation}
E(t) = \alpha e^{-\beta t} + \gamma
\end{equation}
where ${\alpha, \beta, \gamma}$ characterize learning dynamics. The superior performance of $\tau=1.0$ emerges from optimal balance between gradient preservation and noise suppression, particularly evident during early training ($t < 20$) where gradient magnitudes peak.
These findings collectively establish our framework's effectiveness in balancing privacy preservation with model utility. The empirically determined parameters ($k=10$, gradient-based weighting, $\tau=1.0$) represent optimal configuration points in the privacy-utility landscape, enabling robust trajectory estimation while maintaining strong privacy guarantees. This careful parameter calibration demonstrates that properly designed privacy mechanisms can coexist with high-accuracy inertial navigation, opening new possibilities for secure, privacy-aware positioning systems.
\subsubsection{Comparative Analysis on Data Augmentation}
Figure~\ref{fig:ablation}(d) \hl{presents a comparative evaluation of data augmentation techniques across OxIOD, RIDI, RoNIN, and Mech-IO. The baseline configuration \emph{Without Augmentation} consistently resulted in the highest trajectory errors, evidencing limited resilience to distributional shifts. Incorporating Gaussian noise and random horizontal rotation} \cite{deepils} \hl{offered modest gains, reducing error by approximately 10--15\%. 
In contrast, the proposed augmentation strategy, combining domain-specific transformations and sensor noise emulation, yielded substantial improvements across all datasets. Notably, it achieved reductions to 1.63\,m on OxIOD and 3.92\,m on RoNIN, corresponding to up to 25\% relative error decrease}. 

\section{Conclusion}
This paper introduced ConvXformer, a privacy-preserving hybrid architecture that synergistically combines ConvNeXt blocks with Transformer encoders for robust inertial navigation. Its hierarchical design achieved superior performance on benchmark datasets, improving positioning accuracy by up to 58.2\% over state-of-the-art methods. This high utility is maintained even under strong privacy guarantees, enabled by a novel gradient-aligned noise injection (GANI) mechanism. Coupled with adaptive clipping and truncated SVD, our privacy framework outperforms traditional approaches by 25-30\% while exhibiting only a 21.3\% average performance reduction compared to its non-private counterpart. The framework's robustness was validated in real-world conditions on the newly introduced Mech-IO dataset, which features challenging magnetic environments. Future work will focus on extending the architecture to multi-modal sensor fusion and developing more efficient mechanisms for resource-constrained devices. ConvXformer's success demonstrates that by co-designing for privacy and performance, it is possible to create secure, high-accuracy indoor positioning systems that protect user data.

\bibliographystyle{IEEEtran}
\bibliography{main}

\end{document}